\documentclass{article} % For LaTeX2e
\usepackage{iclr2025_conference,times}

\usepackage{amsmath,amsfonts,bm}

\def\eqref#1{equation~\ref{#1}}
\def\1{\bm{1}}

\DeclareMathAlphabet{\mathsfit}{\encodingdefault}{\sfdefault}{m}{sl}
\SetMathAlphabet{\mathsfit}{bold}{\encodingdefault}{\sfdefault}{bx}{n}

\usepackage{hyperref}
\usepackage{url}

\usepackage{booktabs}       % professional-quality tables
\usepackage{amsfonts}       % blackboard math symbols
\usepackage{nicefrac}       % compact symbols for 1/2, etc.
\usepackage{microtype}      % microtypography
\usepackage{xcolor}         % colors
\usepackage{graphicx}
\usepackage{wrapfig}
\usepackage{mathtools}
\usepackage{amsthm}
\usepackage{enumitem}
\usepackage{subcaption}
\usepackage{natbib}
\usepackage[capitalize,noabbrev]{cleveref}
\usepackage{algorithm}
\usepackage{algpseudocode}
\usepackage{multirow}
\usepackage{soul}    % 취소선을 넣기위해 사용되는 일시적인 패키지, st{text}를 하면, text에 취소선을 그을 수 있음. 

\theoremstyle{plain}
\newtheorem{theorem}{Theorem}[section]
\newtheorem{proposition}[theorem]{Proposition}
\newtheorem{lemma}[theorem]{Lemma}

\theoremstyle{definition}

\newtheorem{remark}[theorem]{Remark}

\providecommand{\jm}{\textcolor{black}}

\title{Scalable Simulation-free Entropic Unbalanced Optimal Transport}

\author{Jaemoo Choi\\
Georgia Institute of Technology\\
\texttt{jchoi843@gatech.edu} \\
\And
Jaewoong Choi\thanks{Corresponding Author} \\
Korea Institute for Advanced Study \\
\texttt{chjw1475@kias.re.kr} \\
}

\iclrfinalcopy % Uncomment for camera-ready version, but NOT for submission.
\begin{document}

\maketitle

\begin{abstract}
The Optimal Transport (OT) problem investigates a transport map that connects two distributions while minimizing a given cost function. Finding such a transport map has diverse applications in machine learning, such as generative modeling and image-to-image translation. In this paper, we introduce a scalable and simulation-free approach for solving the Entropic Unbalanced Optimal Transport (EUOT) problem. We derive the dynamical form of this EUOT problem, which is a generalization of the Schrödinger bridges (SB) problem. Based on this, we derive dual formulation and optimality conditions of the EUOT problem from the stochastic optimal control interpretation. By leveraging these properties, we propose a simulation-free algorithm to solve EUOT, called Simulation-free EUOT (SF-EUOT). While existing SB models require expensive simulation costs during training and evaluation, our model achieves simulation-free training and one-step generation by utilizing the reciprocal property. Our model demonstrates significantly improved scalability in generative modeling and image-to-image translation tasks compared to previous SB methods.
\end{abstract}

\section{Introduction}
\label{sec:intro}

The distribution transport problem investigates finding a transport map that bridges one distribution to another. The distribution transport problem has various applications in machine learning, such as generative modeling  \citep{otm, uotm, uotmsd, jko, imf}, image-to-image translation \citep{unsb}, and biology \citep{cellot, population}. Optimal Transport (OT) \citep{ComputationalOT, villani} explores the most cost-efficient transport map among them. For discrete measures, the OT map can be computed exactly through convex optimization, but it is computationally expensive. In contrast, the Entropic Optimal Transport (EOT) problem presents strict convexity and can be computed more efficiently through the Sinkhorn algorithm \citep{sinkhorn, EOT}.
For continuous measures, several machine learning approaches have been proposed for learning the EOT problem \citep{ipf, imf, enotdiffusion}. These approaches typically utilize a dynamic version of EOT, known as the Schrödinger Bridge problem \citep{sb2, imf}. 
% However, these methods still face challenges in scalability when handling high-dimensional and complex distributions, such as image datasets.

The Schrödinger Bridge (SB) is a finite-time diffusion process that bridges two given distributions while minimizing the KL divergence to a reference process \citep{sb1, sb2}. Various works for solving this SB problem have been proposed \citep{sb_toy1, sb_toy2, bortoli2021diffusion, ipf, imf}. However, these methods tend to exhibit scalability challenges when the source and target distributions are in high-dimensional spaces and the distance between them is large. Consequently, these approaches tend to be limited to low-dimensional datasets \citep{sb_toy1, sb_toy2} or rely on a pretraining process for generative modeling task \citep{ipf, imf}. Moreover, existing approaches require simulation of the diffusion process, leading to excessive training and evaluation costs \citep{imf, ipf, enotdiffusion}.

In this paper, we propose an algorithm for solving the Entropic Unbalanced Optimal Transport (EUOT) problem, called \textit{\textbf{Simulation-free EUOT (SF-EUOT)}}. The EUOT problem generalizes the EOT problem by relaxing the precise matching of the target distribution into soft matching through $f$-divergence minimization. Depending on the chosen divergence measure, the EUOT problem can encompass the EOT problem.
To solve this, our model is based on the stochastic optimal control interpretation of the dual formulation of EUOT. Compared to previous works, our model introduces a novel parametrization consisting of the static generator for path measure $\rho$ and the time-dependent value function $V$. This parametrization enables simulation-free training and one-step generation through the reciprocal property (Sec. \ref{sec:em_alg}). Our model achieves significantly improved scalability compared to existing SB models. 
Specifically, our model achieves a FID score of 3.02 with NFE 1 on CIFAR-10 without pretraining, which is comparable to the state-of-the-art results of the SB model with pretraining. For instance, IPF \citep{ipf} achieves a FID score of 3.01 with NFE 200 and IMF \citep{imf} achieves 4.51 with NFE 100. 
Furthermore, our model outperforms several OT models on image-to-image translation benchmarks. Our contributions can be summarized as follows:
\begin{itemize}
    \item We derive the dynamical form of the Entropic Unbalanced Optimal Transport (EUOT) problem, which is a generalization of the Schrödinger Bridge Problem. 
    \item We derive the dual formulation of the EUOT problem and its optimality conditions through the stochastic optimal control interpretation (Sec. \ref{sec:dynamic_dual_euot}).
    \item We propose an efficient method for solving  EUOT problem based on these interpretations. Our model offers simulation-free training by utilizing the reciprocal property (Sec. \ref{sec:em_alg}).
    \item Our model greatly improves the scalability of EUOT models. To the best of our knowledge, our model is the first EOT model that presents competitive results without pretraining (Sec. \ref{sec:exp}).
\end{itemize}

\paragraph{Notations and Assumptions}
Let $\mathcal{X} = \mathbb{R}^d$ where $d$ is a data dimension. 
Let $\mathcal{M}_2 (\mathcal{X})$ be a set of positive Borel measures with finite second moment.
Moreover, let $\mathcal{P}_2(\mathcal{X})$ be a set of probability densities in $\mathcal{M}_2 (\mathcal{X})$.
Let $\Phi_2(\mathcal{X})$ be a set of continuous functions $\varphi : \mathcal{X}\rightarrow \mathbb{R}$ such that $|\varphi(x)| \leq a+ b\lVert x\rVert^2 \ \text{for all} \ x\in \mathcal{X},$ for some $a,b >0$. 
Let $\Phi_{2,b}$ be a set of bounded-below functions in $\Phi_2$.
Throughout the paper, let $\mu,\nu \in \mathcal{P}_2(\mathcal{X})$ be the absolutely continuous source and target distributions, respectively.
In generative modeling tasks, $\mu$ and $\nu$ correspond to \textit{tractable noise} and \textit{data distributions}.
$\Pi(\mu, \nu)$ denotes the set of joint probability distributions on $\mathcal{P}_2(\mathcal{X}\times \mathcal{X})$ whose marginals are $\mu$ and $\nu$.
$W_t$ represents the standard Wiener process on $\mathcal{X}$.

\section{Background}
In this section, we provide a brief overview of key concepts in the OT theory and the Schr\"{o}dinger Bridge (SB) problem. For detailed discussion on related works, please refer to Appendix \ref{appen:relatedworks}.

\subsection{Optimal Transport Problems}
\paragraph{Kantorovich's Optimal Transport (OT)}
Kantorovich's Optimal Transport (OT) \citep{Kantorovich1948} problem addresses the problem of searching for the most cost-effective way to transform the source distribution $\mu$ to the target distribution $\nu$. Formally, it can be expressed as the following minimization problem:
\begin{equation} \label{eq:kantorovichOT}
    \inf_{\pi \in \Pi(\mu, \nu)} \int \frac{1}{2} \| x-y \|_{2}^2 d\pi(x,y)   
\end{equation}
Here, we consider the quadratic cost $ \frac{1}{2} \| x-y \|_{2}^{2}$. For an absolutely continuous $\mu \in \mathcal{P}_2(\mathcal{X})$, the optimal transport plan $\pi^{\star}$ exists. Moreover, the optimal transport is deterministic. In other words, there exists a deterministic OT map $T^{\star}:\mathcal{X} \rightarrow \mathcal{X}$ such that $(Id \times T^{\star})_{\#} \mu  = \pi^{\star}$, i.e. $\pi^{\star}(\cdot|x) = \delta_{T^{\star}(x)}$ is a delta measure.

\paragraph{Entropic Optimal Transport (EOT)}
The Entropic Optimal Transport (EOT) problem is an entropy-regularized version of the OT problem, which introduces an entropy term for the coupling $\pi$ to the standard OT problem.
Formally, EOT can be written as the following minimization problem:
\begin{equation} \label{eq:eot}
    \inf_{\pi \in \Pi(\mu, \nu)} \int_{\mathcal{X}\times \mathcal{X}} \frac{1}{2} \lVert x - y \rVert^2 d\pi (x,y) - \sigma^2 H(\pi),
\end{equation}
where $H(\pi)$ denotes the entropy of $\pi$ and $\sigma > 0$.
The optimal transport plan $\pi^\star$ of Eq. \ref{eq:eot} is unique due to the strict convexity of $H(\pi)$.
Moreover, when $\sigma >0$, the optimal transport is a stochastic map, i.e. $\pi^\star(\cdot|x)$ is a stochastic map.

\subsection{Schr\"{o}dinger Bridge Problem and its Properties} \label{sec:background_sb}
In this section, we discuss the properties of the Schr\"{o}dinger Bridge (SB) problem \citep{sb1, sb2}. First, we describe the reciprocal property of SB \citep{reciprocal}. Then, we introduce the equivalence between the SB problem and EOT \citep{sb1, sb2}. Finally, we present the dual formulation of EOT \citep{enotdiffusion}. These properties will be extended to the Entropic Unbalanced Optimal Transport (EUOT) in Sec. \ref{sec:dynamic_dual_euot}.

\paragraph{Schr\"{o}dinger Bridge (SB) with Wiener prior}
Let $\Omega := [0,1] \times \mathcal{X}$, where $t\in [0,1]$ represents the time variable. Throughout this paper, we denote the probability density induced by the following stochastic process $\{X^u_t\}$ as $\mathbb{P}^u \in \mathcal{P}_2(\Omega)$. In other words, $\mathbb{P}^u_t$ represents the distribution of the stochastic process $\{X^u_t\}$:
\begin{equation} \label{eq:sto_process}
    dX^u_t = u(t,X^u_t) dt + \sigma dW_t, \quad \ X^u_0 \sim \mu.
\end{equation}
with drift $u:\Omega \rightarrow \mathcal{X}$, diffusion term $\sigma > 0$, and initial distribution $\mu$.
Moreover, let $\mathbb{Q}$ be the Wiener process with a diffusion term of $\sigma$.
Then, the SB problem aims to find the probability density $\mathbb{P}^u$ that is most close to the Wiener process $\mathbb{Q}$.
Formally, the SB problem is defined as follows:
\begin{equation} \label{eq:SB_KL}
    \inf_u D_{\text{KL}}(\mathbb{P}^u | \mathbb{Q}) \quad
    {\rm s.t.} \quad \mathbb{P}^u_0 = \mu, \ \mathbb{P}^u_1 = \nu \qquad
    {\rm where} \quad d\mathbb{Q}_t = \sigma dW_t, \, \ \mathbb{Q}_0 \sim \mu.
\end{equation}
Then, we can express this SB problem in terms of the drift $u$ and the path measure $\{\rho_t\}_{t\in [0,1]}$ of the stochastic process $X^u_t$. By Girsanov's theorem \citep{girsanov}, 
\begin{equation} \label{eq:girsanov}
     D_{\text{KL}}(\mathbb{P}^u | \mathbb{Q}) = \frac{1}{\sigma^2} \, \mathbb{E} \left[ \int_0^1 \frac{1}{2} \| u_t(X^u_t) \|^2 dt \right].
\end{equation}
Moreover, by the Fokker-Planck equation \citep{fokker-planck}, $\{\rho_t\}_{t\in [0,1]}$ satisfies the following equation:
\begin{equation} \label{eq:Fokker-Planck}
     \partial_t \rho_t + \nabla \cdot (u_t \rho_t ) - \frac{\sigma^2}{2} \Delta \rho_t = 0, \quad \rho_0 = \mu.
\end{equation}
Therefore, by combining Eq. \ref{eq:girsanov} and \ref{eq:Fokker-Planck}, the SB problem (Eq. \ref{eq:SB_KL}) can be reformulated as follows:
\begin{align} \label{eq:SB_dynamic}
        \inf_u \int_0^1 \int_{\mathcal{X}} \frac{1}{2} \| u_t(x) \|^2 d\rho_t (x) dt, \quad
        {\rm s.t.} \quad \partial_t& \rho_t + \nabla \cdot (u_t \rho_t ) - \frac{\sigma^2}{2} \Delta \rho_t = 0, \,\, \rho_0 = \mu, \rho_1 = \nu.
\end{align}

\paragraph{Reciprocal Property of SB}
In this paragraph, we provide an intuition of the Reciprocal Property \citep{reciprocal, sb1, imf} of SB.
To begin with, the minimization objective of the SB problem in Eq. \ref{eq:SB_KL} can be decomposed into the KL divergence of the joint distribution between $t=0$ and $t=1$, and the conditional KL divergence. 
Formally, the decomposition is written as follows:
\begin{equation} \label{eq:decomposition_KL}
 D_{\text{KL}}(\mathbb{P}^u | \mathbb{Q}) = D_{\text{KL}}(\mathbb{P}^u_{0,1} | \mathbb{Q}_{0,1}) + \int_{\mathcal{X}\times \mathcal{X}} \int_0^1 D_{\text{KL}}(\mathbb{P}^u_t(\cdot|x,y) | \mathbb{Q}_t (\cdot|x,y)) \, dt \, d\mathbb{P}^u_{0,1}(x,y),   
\end{equation}
where $\mathbb{P}^u_{0,1} \in \Pi(\mu, \nu)$ denotes the joint distribution on $t=0$ and $t=1$ induced by $\mathbb{P}^u$.
Let $\mathbb{P}^\star$ be the optimal solution for the LHS of Eq. \ref{eq:decomposition_KL}.
Surprisingly, \citet{reciprocal} discovered that the last term on the RHS of Eq. \ref{eq:decomposition_KL} is zero for $\mathbb{P}^\star$, i.e., $\mathbb{P}^{\star}_t (\cdot| x,y) = \mathbb{Q}_t (\cdot| x,y)$ for $(x,y) \sim \mathbb{P}^{\star}_{0,1}$ almost surely. \textbf{This property allows us to characterize the bridge measure between $X_{0}^{u}=x$ and $X_{1}^{u}=y$.}
\begin{equation} \label{eq:reciprocal}
    \mathbb{P}^{\star}_t (\cdot|x,y) = \mathbb{Q}_t (\cdot| x,y)= \mathcal{N}(\cdot| (1-t)x+ty, \sigma^2 t (1-t) I).
\end{equation}
We refer to Eq. \ref{eq:reciprocal} as the \textit{reciprocal property} of SB. This property will be utilized in our neural network parametrization in Sec. \ref{sec:em_alg}.

\paragraph{Equivalence between SB and EOT} 
The reciprocal property establishes equivalence between SB and EOT (For rigorous explanation, see \citep{sb2, sb1}). By definition, the KL divergence between two static transport plans is given as follows:
\begin{equation} \label{eq:def_kl}
    \sigma^2 D_{\text{KL}}(\pi | \mathbb{Q}_{0,1}) = \int_{\mathcal{X}\times \mathcal{X}} \frac{1}{2} \lVert x - y \rVert^2 d\pi (x,y) - \sigma^2 H(\pi).
\end{equation}
Here, the reciprocal property implies that $D_{\text{KL}}(\mathbb{P}^{\star} | \mathbb{Q}) = D_{\text{KL}}(\mathbb{P}^{\star}_{0,1} | \mathbb{Q}_{0,1})$ in Eq. \ref{eq:decomposition_KL}. Combining this with Eq. \ref{eq:def_kl}, we can show that the dynamical SB problem (Eq. \ref{eq:SB_KL}) is equivalent to the static EOT problem (Eq. \ref{eq:eot}):
% Putting the reciprocal property (Eq. \ref{eq:reciprocal}) and Eq. \ref{eq:def_kl} together, we can induce that the SB problem (Eq. \ref{eq:SB_KL}) is equivalent to EOT (Eq. \ref{eq:eot}):
\begin{equation}
    \sigma^2 D_{\text{KL}}(\mathbb{P}^{\star} | \mathbb{Q}) = \inf_{\pi \in \Pi(\mu, \nu)} \left[ \int \frac{1}{2} \| x-y \|^2 d\pi(x,y) - \sigma^2 H(\pi) \right] \quad {\rm where} \quad \pi^{\star} = \mathbb{P}^{\star}_{0,1}.
\end{equation}

\paragraph{Dual Form of EOT}
Additionally, we introduce a dual form of the EOT problem presented in \citet{enotdiffusion}. As described in Eq. \ref{eq:def_kl}, the entropy-regularized minimization objective in EOT can be interpreted as minimizing $D_{\text{KL}} (\pi|\mathbb{Q}_{0,1})$ with respect to the joint distribution $\pi \in \Pi(\mu,\nu)$. Inspired by the weak OT theory \citep{weakOT}, \citet{enotdiffusion} derived the following dual form:
\begin{equation}
    \inf_{\pi\in\Pi(\mu,\nu)} D_{\text{KL}} (\pi|\mathbb{Q}_{0,1}) = \sup_{V\in \Phi_{2,b}} \left[ \int_{\mathcal{X}} V^C(x)d\mu(x) - \int_{\mathcal{X}} V(y) d\nu(y) \right],
\end{equation}
where 
$
    V^C(x) := \inf_{\rho_1 \in \mathcal{P}_2(\mathcal{X})}\left[ D_{\text{KL}}\left(\rho_1 \ | \ \mathbb{Q}_{1|0}(\cdot|x) \right) +\int_{\mathcal{X}} V(y) \nu(y) \right].
$
By applying the Girsanov theorem (Eq. \ref{eq:girsanov}) to the KL divergence in $V^C$, \citet{enotdiffusion} obtain the following problem:
\begin{align} \label{eq:weakot_eot}
    \begin{split}
        \sup_{V\in \Phi_{2,b}} &\inf_{u} \left[ \mathbb{E} \left[\int_0^1 \frac{1}{2} \| u_t(X^u_t) \|^2 dt + V(X^u_1) \right] -  \int_{\mathcal{X}} V(y) d\nu(y) \right], \ {\rm s.t.} \ X^u_0 \sim \mu, \ X^u_1 \sim \nu.
    \end{split}
\end{align}

\section{Dynamcial and Dual Form of Entropic Unbalanced Optimal Transport}
\label{sec:dynamic_dual_euot}

In this section, \textbf{we derive various formulations of the \textit{Entropic Unbalanced Optimal Transport (EUOT)}.}
First, we introduce the EUOT problem and derive its dynamical formulation (Theorem \ref{theorem:dynamic_EUOT}).
The dynamical formulation of EUOT encompasses the Schrödinger Bridge (SB) problem, which is a dynamical form of EOT (Sec. \ref{sec:background_sb}).
% Here, the EUOT problem is an extension of the EOT problem. In this regard, the dynamical formulation of EUOT encompasses the Schrödinger Bridge (SB) problem, which is a dynamical form of EOT (Sec. \ref{sec:background_sb}).
Next, we derive the dynamical dual form of EUOT by leveraging the SB theory (Theorem \ref{theorem:weakOTdual}).
The dynamical dual will play a crucial role in deriving our learning objective and we will utilize the reciprocal property from the dynamical for our simulation-free parametrization in Sec \ref{sec:em}.
For detailed proof of the theorems, refer to Appendix \ref{appen:proofs}.

% \subsection{Entropic Unbalanced Optimal Transport and its Dynamical Form} \label{sec:dynamic_euot}
\paragraph{Entropic Unbalanced Optimal Transport (EUOT)}
In this paper, we consider the EUOT problem with a fixed source measure constraint, i.e., $\pi_{0}=\mu$ (Eq. \ref{eq:euot}). While EOT assumes precise matching of two measures, i.e., $\pi_{0}=\mu, \pi_{1}=\nu$ (Eq. \ref{eq:eot}), EUOT relaxes the marginal constraint using the $f$-divergence $D_\Psi$. This unbalanced variant of the optimal transport problem offers outlier robustness \citep{robust-ot, uotm} and the ability to handle class imbalance in datasets \citep{eyring2024unbalancedness}.
Formally, the EUOT problem is defined as follows:
\begin{equation} \label{eq:euot}
    \inf_{\pi_0 = \mu, \pi \in \mathcal{P}_2 (\mathcal{X}\times \mathcal{X})} \left[ \int_{\mathcal{X}\times \mathcal{X}} \frac{1}{2} \lVert x - y \rVert^2 d\pi (x,y) - \sigma^2 H(\pi) + \alpha D_\Psi (\pi_1 | \nu) \right],
\end{equation}
where $\alpha > 0$ denotes the divergence penalization intensity and $\Psi:[0,\infty)\rightarrow [0,\infty]$ is assumed to be a convex, lower semi-continuous, non-negative function, and $\Psi(1)=0$.
We call $\Psi$ an \textit{entropy function} of $D_\Psi$.
Furthermore, due to the convexity of $\Psi$, the solution $\pi^\star$ of Eq. \ref{eq:euot} is unique. 

\begin{remark}[\textbf{EUOT is a Generalization of EOT}] \label{remark_euot}
    Suppose $\Psi$ is a convex indicator $\iota$ at $\{1\}$, i.e. $\iota(x) = 0$ if $x=1$ and $\iota(x) = \infty$ otherwise. Then, the corresponding $f$-divergence $D_\iota(\rho_1 | \rho_2)$ equals 0 if $\rho_1 = \rho_2$ almost surely and $\infty$ otherwise. To obtain a finite objective in Eq. \ref{eq:euot}, it must hold $\pi_1 = \nu$ almost surely. Therefore, when $\Psi = \iota$, EUOT becomes EOT.
\end{remark}

\paragraph{Dynamical Formulation of EUOT}
The following theorem proves the \textbf{dynamical formulation of the EUOT problem}. Note that, compared to the dynamical formulation for EOT (SB) (Eq. \ref{eq:SB_dynamic}), the minimization objective (Eq. \ref{eq:dynamic_euot}) contains an additional divergence penalization term. Moreover, the theorem states that the optimal path measure $\mathbb{P}^\star$ for Eq. \ref{eq:dynamic_euot} satisfies the reciprocal property (Eq. \ref{eq:reciprocal_euot}), as in the SB problem. 
\textbf{This reciprocal property will be leveraged for the simulation-free parametrization} in Sec \ref{sec:em}.
\begin{theorem} \label{theorem:dynamic_EUOT}
    The EUOT problem is equivalent to the following dynamical transport problem:
    \begin{align} \label{eq:dynamic_euot}
        \begin{split}
            &\inf_u \left[ \int_0^1 \int_{\mathcal{X}} \frac{1}{2}\| u(t,x) \|^2 d\rho_t(x) dt + \alpha D_\Psi (\rho_1| \nu)\right],
        \end{split}
    \end{align}
    where $\partial_t \rho_t + \nabla \cdot (u_t \rho_t ) - \frac{\sigma^2}{2} \Delta \rho_t = 0$ and $\rho_0 = \mu$.
    Moreover, the optimal solution $\mathbb{P}^\star$ satisfies the reciprocal property, i.e., 
    \begin{equation} \label{eq:reciprocal_euot}
        \mathbb{P}^{\star}_t (\cdot|x,y) = \mathcal{N}(\cdot| (1-t)x+ ty, \sigma^2 t (1-t) I), \quad (x,y) \sim \mathbb{P}^\star_{0,1} \text{-almost surely.}
    \end{equation}
\end{theorem}
\begin{remark}[\textbf{Dynamic form of EUOT is a Generalization of SB}]
    As discussed in Remark \ref{remark_euot}, suppose the entropy function $\Psi$ is the convex indicator $\iota$. By the same argument, $\rho_1 = \nu$ almost surely. Therefore, the dynamical EUOT (Eq. \ref{eq:dynamic_euot}) becomes equivalent to the SB (Eq. \ref{eq:SB_dynamic}) when $\Psi = \iota$. 
\end{remark}

\paragraph{Dynamical Dual form of EUOT}
By applying the Fenchel-Rockafellar theorem \citep{fenchel-rockafellar} and leveraging the dual form in \citet{enotdiffusion}, we derive the following dynamical dual formulation of EUOT:
\begin{proposition}[\textbf{Dual formulation of EUOT}] \label{theorem:weakOTdual}
    The static EUOT problem (Eq. \ref{eq:euot}) is equivalent to the following problem:
    \begin{align} \label{eq:weakot_euot}
        \begin{split}
            \sup_{V\in \Phi_{2,b}} \inf_{u} \left[ \mathbb{E} \left[\int_0^1 \frac{1}{2} \| u(t,X^u_t) \|^2 dt + V(X^u_1) \right] -  \int_{\mathcal{X}} \alpha \Psi^* \left( \frac{V(y)}{\alpha} \right) d\nu(y) \right].
            % , \quad {\rm s.t.} \ X_0 \sim \mu.
        \end{split}
    \end{align}
    where $dX^u_t = u(t,X^u_t) dt + \sigma dW_t$ and $X^u_0 \sim \mu$.
\end{proposition}

Note that when the entropy function $\Psi$ is a convex indicator $\iota$, the convex conjugate $\Psi^*(x)=x$. In this case, the dynamical dual form of EUOT (Eq. \ref{eq:weakot_euot}) reduces to Eq. \ref{eq:weakot_eot}. Therefore, Prop. \ref{theorem:weakOTdual} is an extension of the dual form of EOT. This dual formulation and its interpretation via stochastic optimal control will be utilized to derive our learning objective in Sec \ref{sec:em}.

\begin{algorithm}[t]
\caption{Simulation-free EUOT Algorithm}
\begin{algorithmic}[1]
\Require The source distribution $\mu$ and the target distribution $\nu$. Convex conjugate of entropy function $\Psi^*$. Generator network $T_\theta$ and the value network $v_\phi:[0,1]\times \mathcal{X} \rightarrow \mathbb{R}$. Time (sampling) distribution $\mathcal{T}$. Total iteration number $K$.
\For{$k = 0, 1, 2 , \dots, K$}
    \vspace{1pt}
    \State Sample a batch $x \sim \mu$, $y \sim \nu$, $t\sim \mathcal{T}$, $z, \eta_1, \eta_2 \sim \mathcal{N}(\mathbf{0}, \mathbf{I})$.
    \vspace{2pt}
    \State $\hat{y}\leftarrow T_\theta (x, z)$
    \vspace{2pt}
    \State $x_t \leftarrow (1-t) x + t \hat{y} + \sigma \sqrt{t(1-t)} \eta_1$.
    \vspace{2pt}
    \State $u_t \leftarrow \frac{\hat{y} - x_t}{1 - t}$, $\sigma_t \leftarrow \sigma \sqrt{\frac{(1-t-\Delta t)\Delta t}{1-t}}$.
    \vspace{3pt}
    \State $x_{t+\Delta t} \leftarrow x_t + u_t \Delta t +  \sigma_t  \eta_2$.
    \vspace{3pt}
    \State $R \leftarrow \frac{v_\phi(t+\Delta t, x_{t+\Delta t}) - v_\phi(t, x_t)}{\Delta t} -  \frac{\alpha}{2} \lVert \nabla v_\phi(t, x_t) \rVert^2 + \frac{\sigma^2}{2} \Delta v_\phi (t,x)$.
    \vspace{3pt}
    \State Update $v_\phi$ by the following loss: $\lambda_D \lVert R \rVert^p - \frac{\alpha}{2} \lVert \nabla v_\phi (t, x_t) \rVert^2 - v_\phi (1, \hat{y}) + \Psi^* (v_\phi (1, y)).$
    \State Update $T_\theta$ by the following loss: $ \lambda_G R $.
\EndFor
\end{algorithmic}
\label{alg:em_euot}
\end{algorithm}

\section{Method} \label{sec:em_alg}
In this section, we propose our scalable and simulation-free algorithm for solving the EUOT problem. In Sec \ref{sec:soc_euot}, we reinterpret our dual form (Eq. \ref{eq:weakot_euot}) through various formulations. Specifically, starting with the stochastic optimal control (SOC) interpretation, we interpret Eq. \ref{eq:weakot_euot} as a bilevel optimization problem, involving the time-dependent value function $V$ and the path measure of $X_{t}^{u}$, $\rho:[0,1]\times \mathcal{X} \rightarrow \mathbb{R}$. In Sec. \ref{sec:em}, we derive our \textbf{Simulation-free EUOT} algorithm by integrating the results from Sec. \ref{sec:soc_euot} and the reciprocal property in Theorem \ref{theorem:dynamic_EUOT}.

% \subsection{Stochastic Optimal Control Viewpoint of EUOT} 
\subsection{Stochastic Optimal Control Interpretation of Dual Form of EUOT} \label{sec:soc_euot}
\paragraph{SOC Interpretation of Inner-loop Problem}
First, we investigate the inner optimization with respect to $u$ in the dual form of EUOT (Eq. \ref{eq:weakot_euot}). By applying the SOC interpretation to this inner problem, we extend the value function $V: \mathcal{Y} \rightarrow \mathbb{R}$ to be time-dependent.
The inner optimization of the dual form of EUOT is given as follows:
\begin{align} \label{eq:soc_euot}
    \begin{split}
        \inf_u \mathbb{E} \left[ \int_0^1 \frac{1}{2}\| u(t,X^u_t) \|^2 dt + V (X^u_1) \right].
    \end{split}
\end{align}
This optimization problem can be regarded as a stochastic optimal control (SOC) problem \citep{soc} by regarding $X^u_t$ as the controlled SDE and $V(\cdot)$ as the terminal cost.
Now, let $\Tilde{V}:[0,1]\times \mathcal{X} \rightarrow \mathbb{R}$ be the \textbf{value function} of this problem, i.e.
\begin{equation}
    \Tilde{V}(t, x) = \inf_u \mathbb{E} \left[ \int_t^1 \frac{1}{2} \lVert u_t \rVert^2 dt +V(X^u_1) \,\, \big| \,\, X^u_t = x \right].
\end{equation}
Then, the inner minimization problem (Eq. \ref{eq:soc_euot}) can be reformulated by the HJB equation for this time-dependent value function $\Tilde{V}$ \citep{highdimhjb, hjb}:
\begin{equation} \label{eq:hjb}
    \partial_t \Tilde{V}_t -\frac{1}{2} \| \nabla \Tilde{V}_t \|^2 + \frac{\sigma^2}{2} \Delta \Tilde{V}_t = 0, \quad \Tilde{V}_1 = V \quad \text{ and } \quad u^\star = - \nabla \Tilde{V}.
\end{equation} 
% where the optimal control $u^\star$ in Eq. \ref{eq:soc_euot} is given as $u^\star = - \nabla \Tilde{V}$.
where $u^\star$ denotes the optimal control in Eq. \ref{eq:soc_euot}.
Since $\Tilde{V}_1 := \Tilde{V}(1,\cdot) = V$, we can say $\Tilde{V}$ is an time-dependent extension of $V$.
% \jm{For simplicity, from now on, we refer $V:[0,1]\times \mathcal{X} \rightarrow \mathbb{R}$ as the time-dependent value function.}
For simplicity, we will denote the time-dependent value function $\Tilde{V}$ using the same notation as the original value function, specifically $V:[0,1]\times \mathcal{X} \rightarrow \mathbb{R}$.

% \paragraph{Interpretation of the Dual form}
\paragraph{Reinterpretation of Dual Form}
We begin by interpreting the dual form of EUOT (Eq. \ref{eq:weakot_euot}) as a bilevel optimization problem. \textbf{Our goal is to represent this optimization problem concerning the value function $V$ and the path measure $\rho_{t}$.}
By splitting Eq. \ref{eq:weakot_euot} into inner-minimization and outer-maximization, we arrive at the following optimization problem:
\begin{multline} \label{eq:dual_interpret1}
    % \begin{split}
    \sup_{V\in \Phi_{2,b}} \mathbb{E} \left[ \int^1_0 \frac{1}{2} \Vert u^\star_t (X^u_t) \Vert^2  dt +  V_1(X^u_1) \right] -  \int_{\mathcal{X}} \alpha \Psi^* \left( \frac{V_1(y)}{\alpha} \right) d\nu(y), \\ \text{ s.t. } u^\star = \arg\inf_{u}  \mathbb{E} \left[\int_0^1 \frac{1}{2} \| u_t(X^u_t) \|^2 dt + V_1(X^u_1) \right],    
    % \end{split}
\end{multline}
where $dX^u_t = u(t,X^u_t) dt + \sigma dW_t$ with $\, X^u_0 \sim \mu$. 
% Then, we can represent Eq. \ref{eq:dual_interpret1} in terms of the $\rho_{t} = \text{Law}(X_{t}^{u})$ by using the Fokker-Plank equation as follows:
Then, we can rewrite $X_{t}^{u}$ using $\rho_{t} = \text{Law}(X_{t}^{u})$ by the Fokker-Plank equation:
\begin{align} 
    \sup_{V\in \Phi_{2,b}} \int^1_0 \int_{\mathcal{X}} \frac{1}{2} \Vert u^\star_t \Vert^2 \rho^\star_t dx dt + \mathbb{E}_{\hat{y} \sim \rho^\star_1} \left[ V_1(\hat{y}) \right] -  \int_{\mathcal{X}} \alpha \Psi^* \left( \frac{V_1(y)}{\alpha} \right) d\nu(y). \\
    \text{s.t. } (u^\star, \rho^\star) =  \arg\inf_{(u, \rho)} \mathbb{E}_{(t,x)\sim \rho} \left[ \frac{1}{2} \lVert u_t(x) \rVert^2 \right] + \mathbb{E}_{\hat{y}\sim \rho_1}\left[ V_1(\hat{y}) \right]. \label{eq:dual_interpret2}
\end{align}
where $\partial_t \rho + \nabla \cdot (u\rho) - (\sigma^2 / 2) \Delta \rho = 0, \ \rho_0 = \mu$.
% By incorporating the HJB optimality condition from Eq. \ref{eq:hjb} into the inner minimization and using $u^{\star} = - \nabla V$, we can represent the bilevel optimization problem as follows:
Finally, by incorporating the HJB optimality condition (Eq. \ref{eq:hjb}) into (Eq. \ref{eq:dual_interpret2} and using $u^{\star} = - \nabla V$, we can represent the entire bilevel optimization problem above with respect to $V$ and the optimal $\rho_{t}^{\star}$ as follows:
\begin{align} \label{eq:dual_interpret3}
    \begin{split}
    \sup_{V\in \Phi_{2,b}}
    \int^1_0 \int_{\mathcal{X}} \frac{1}{2}  \Vert \nabla V_t \Vert^2 & \rho^\star_t dx dt +
    \mathbb{E}_{\hat{y} \sim \rho^\star_1} \left[ V_1(\hat{y}) \right] -  \int_{\mathcal{X}} \alpha \Psi^* \left( \frac{V_1(y)}{\alpha} \right) d\nu(y). \\
    \text{s.t. } \,\, & \text{(HJB)} \quad \partial_t V_t -\frac{1}{2} \| \nabla V_t \|^2 + \frac{\sigma^2}{2} \Delta V_t = 0 \quad \rho^\star\text{-a.s.}, \\
    &\text{(Fokker-Plank)} \quad \partial_t \rho^\star_t + \nabla \cdot ( -\nabla V_t \rho^\star_t ) - \frac{\sigma^2}{2} \Delta \rho^\star_t = 0,\ \ \rho^\star_0 = \mu..
    \end{split}
\end{align}

Note that we need to derive another optimization problem for the optimal $\rho_{t}^{\star}$ with respect to $V$. This can be achieved by considering the dual form of Eq. \ref{eq:dual_interpret2} as follows:
\begin{equation} \label{eq:dual_interpret4}
    \inf_{\rho} \mathcal{L}_{\rho} \quad \text{ where } \quad \mathcal{L}_{\rho} = \left[ \int^1_0 \int_{\mathcal{X}} \left(\partial_t V_t - \frac{1}{2} \lVert \nabla V_t \rVert^2 + \frac{\sigma^2}{2} \Delta V_t \right) \rho_t dx dt \right].
\end{equation}
For additional details on the derivation Eq. \ref{eq:dual_interpret4}, please refer to Appendix \ref{appen:derivations}.

\subsection{Simulation-free EUOT} \label{sec:em}
In this section, we propose our algorithm for solving the EUOT problem (Eq. \ref{eq:euot}), called \textit{\textbf{Simulation-free EUOT (SF-EUOT)}}. Specifically, we optimize the dynamic value function $V$ and the path measure $\rho_{t}$. Our model is derived from two optimization problems (Eq. \ref{eq:dual_interpret3} and \ref{eq:dual_interpret4}) for $V$ and $\rho_{t}$ presented in Sec \ref{sec:soc_euot}. Additionally, we introduce a simulation-free parametrization of $\rho_{t}$ using the reciprocal property (Theorem \ref{theorem:dynamic_EUOT}), leading to Algorithm \ref{alg:em_euot}. 

\paragraph{Optimization of Value $V$}
We present our loss function for the value function $V$, which is derived from Eq. \ref{eq:dual_interpret3}. Note that the $V$ optimization in Eq. \ref{eq:dual_interpret3} consists of two parts: \textit{the maximization of target functional} and \textit{the HJB optimality conditions}. Therefore, we introduce the following loss function by introducing the HJB condition as a regularization term for the target functional:
\begin{multline} \label{eq:V_objective1}
 \underbrace{\frac{\lambda_D}{\alpha^{p-1}} \int_0^1 \int_{\mathcal{X}} \left\| \partial_t V_t -\frac{1}{2} \| \nabla V_t \|^2 + \frac{\sigma^2}{2} \Delta V_t \right\|^p d\rho_t dt}_{\text{HJB condition}}  \underbrace{\jm{-} \int^1_0 \int_{\mathcal{X}} \frac{1}{2} \Vert \nabla V_t \Vert^2 d\rho_t (x) dt}_{\text{Running cost}} \\ \underbrace{- \int_{\mathcal{X}} V_1 d\rho_1 + \int_{\mathcal{X}} \alpha \Psi^* \left(\frac{V_1}{\alpha}\right) d\nu}_{\text{Terminal cost}},
\end{multline}
where $1\le p \le 2$.
To achieve a simple parametrization, we set $V = \alpha v_\phi$. Then, up to a constant factor, Eq. \ref{eq:V_objective1} can be rewritten as our loss function $\mathcal{L}_{\phi}$ for $V$ as follows:
% Eq. \ref{eq:V_objective1} is rewritten as follows:
\begin{multline} \label{eq:V_objective2}
\mathcal{L}_{\phi} = \lambda_D \int_0^1 \int_{\mathcal{X}} \left\| \partial_t v_\phi -\frac{\alpha}{2} \| \nabla v_\phi \|^2 + \frac{\sigma^2}{2} \Delta v_\phi \right\|^p d\rho_t dt \jm{-} \int^1_0 \int_{\mathcal{X}} \frac{\alpha}{2} \lVert v_\phi \rVert^2 d\rho_t dt \\ - \int_{\mathcal{X}} v_\phi(1,\cdot) d\rho_1 + \int_{\mathcal{X}} \Psi^* \left(v_\phi(1,\cdot) \right) d\nu.
\end{multline}

\paragraph{Optimization of Path Measure $\rho_t$}
By introducing the same parametrization $V = \alpha v_\phi$ as $\mathcal{L}_{\phi}$ into Eq. \ref{eq:dual_interpret4}, we also derive the following loss function $\mathcal{L}_{\theta}$ for $\rho_{t}$:
% The $\rho_{t}$ optimization of Dual II is expressed as follows:
\begin{equation} \label{eq:inner_loop} 
    \mathcal{L}_{\theta} = \inf_{\rho_0 = \mu} \int_0^1 \int_{\mathcal{X}} \left( \partial_t v_\phi -\frac{\alpha}{2} \| \nabla v_\phi \|^2 + \frac{\sigma^2}{2} \Delta v_\phi \right) d\rho_\theta (t,x).
\end{equation}
To minimize Eq. \ref{eq:inner_loop}, it is necessary to obtain the sample $x_t \sim \rho_t$. Here, we exploit the optimality condition for $\rho$, i.e., the reciprocal property, for parametrizing $\rho_{t}$. Specifically, \textbf{we introduce the static generator network $T_{\theta}(x, z)$ for parametrizing $\mathbb{P}^\star_{1|0}(\cdot | x)$ (Eq. \ref{eq:reciprocal_euot})}, where $z \sim \mathcal{N}(0, I)$ is an auxiliary variable. In other words, the conditional transport plan $\pi^\theta (y| x) := \mathbb{P}^\theta_{1|0} (y| x)$ is parametrized as $\pi^\theta (\cdot| x) = T_{\theta}(x,\cdot)_{\#} \mathcal{N}(0, I)$. Then, we can simply obtain $x_t$ by leveraging reciprocal property:
\begin{equation} \label{eq:reciprocal_euot_model}
    x_t = (1-t) x + t \hat{y} + \sigma \sqrt{t(1-t)} \eta, \quad \hat{y} \sim \pi^\theta (\cdot| x), \quad \eta \sim \mathcal{N} (0, I).
\end{equation}

Note that this parametrization also provides an additional benefit. By combining the static generator with the reciprocal property, our model offers \textbf{simulation-free training}. The previous SB models \citep{ipf, imf, enotdiffusion} typically represent the path measure $\rho_{t}$ by parametrizing the drift $u$ with neural network in Eq. \ref{eq:sto_process}. Therefore, sampling from $\rho_{t}$ requires SDE or ODE simulation. However, \textit{our model training does not incur simulation costs because it leverages the reciprocal property.}

\paragraph{Time Discretization}
Following other SB algorithms \citep{ipf, imf, enotdiffusion}, we discretize the time variable $t$. We simply discretize $t$ into uniform intervals of size $\Delta t = 1/N$, where $N$ denotes the number of timesteps. Then, we approximate $\partial_t v_\phi$ using the first-order Euler discretization:
\begin{equation} \label{eq:time_disc}
    \partial_t v_\phi(t,x_t) \approx (v_\phi(t+\Delta t, x_{t+\Delta t})-v_\phi(t, x_{t}))/ \Delta t.
\end{equation}
Moreover, we introduce the discrete probability distribution for $t$ on the set $\{0, \Delta t, 2\Delta t \dots, (N-1)\Delta t \}$ for the Monte-Carlo estimate of $\mathcal{L}_{\phi}$ and $\mathcal{L}_{\theta}$, where $N\Delta t = 1$. The first distribution we consider is the uniform distribution. However, since the terminal cost fluctuates during training, it is advantageous to sample the time variables closer to the terminal time ($t=1$) more frequently. Thus, we also consider $\mathcal{T}(k\Delta t) = \frac{2(k+1)}{N(N+1)}$ for all $k\in \{0, 1, \dots, N-1 \}$. Throughout this paper, we call this distribution as \textit{linear time distribution}.

\section{Experiments}
\label{sec:exp}
In this section, we evaluate our model on various datasets to assess its performance.
In Sec. \ref{sec:exp_image}, we compare our model's scalability on image datasets with other dynamic OT models. 
In Sec. \ref{sec:exp_toy}, we test our model on the synthetic datasets to evaluate whether our model learns the correct EUOT solution.
% Note that the goal of this paper is to \textbf{solve the Entropic Unbalanced Optimal Transport problem} (Eq. \ref{eq:dynamic_euot}). 
Note that the goal of this paper is to \textbf{develop a scalable simulation-free EUOT algorithm} (Eq. \ref{eq:dynamic_euot}).
Therefore, in this section, we compare our model with \textit{(1) the dynamical (Entropic) Optimal Transport generative models}, such as Schrödinger Bridge models \citep{ipf, imf, enotdiffusion} and \textit{(2) the dynamic generative models}, such as diffusion models \citep{scoresde, vahdat2021score} and flow matching \citep{lipman2022flow, tong2024improving}.

\begin{figure}
    \begin{minipage}{1\linewidth}
        \captionof{table}{\textbf{Image Generation on CIFAR-10.}}
            \label{tab:compare-cifar10}
    \centering
    \scalebox{0.75}{
        \begin{tabular}{cccc}
                \toprule
                Class & Model &            FID ($\downarrow$) & NFE    \\ 
                 \midrule
                  \multirow{8}{*}{\textbf{Entropic OT (SB)}}
                  &    \textbf{EUOT-Soft} (Ours)                    &   \textbf{3.02} & 1    \\
                  &    \textbf{EUOT-KL} (Ours)                     &   3.08 & 1    \\
                  &    \textbf{EOT} (Ours)                     &   4.05 & 1    \\ \cmidrule{2-4}
                   &    IPF w/o Pretraining \citep{ipf}                  &   $\geq 100$   & - \\
                  &    IPF w/ Pretraining \citep{ipf}           &   \textbf{3.01}   & 200 \\
                  &    IMF w/ Pretraining \citep{imf}         &   4.51   &   100 \\
                  &    \citet{enotdiffusion}  &   $\geq 100$  & -  \\
                  &    WLF  \citep{wlf}  &    $\geq 100$   & - \\
                \midrule
                \multirow{8}{*}{\textbf{Diffusion}}
                  &  NCSN \citep{song2019generative}&    25.3   & 1000    \\
                  &  DDPM \citep{ddpm}&             3.21  & 1000  \\
                  &  Score SDE (VE) \citep{scoresde} &    2.20  & 2000  \\
                  &  Score SDE (VP) \citep{scoresde}&    2.41  & 2000   \\
                  &  DDIM (50 steps) \citep{ddim}&        4.67 & 50  \\
                  &  CLD \citep{dockhorn2021score} &               2.25  & 2000  \\
                  &  Subspace Diffusion \citep{jing2022subspace} &   2.17  & $\geq$ 1000 \\
                  &  LSGM \citep{vahdat2021score}&         \textbf{2.10}  & 138   \\
                \midrule
                \multirow{2}{*}{\textbf{Flow Matching}}
                  &    FM \citep{lipman2022flow}               &   6.35   & 142  \\
                  &    OT-CFM \citep{tong2024improving}     &   \textbf{3.74}   & 1000  \\
                \bottomrule
                \end{tabular}
            }
    \end{minipage}
    \vspace{-10pt}
\end{figure}

\subsection{Scalability and Scratch Trainability on Image Datasets} \label{sec:exp_image}
% \vspace{-5pt}

\paragraph{Scalability Comparison on Generative Modeling with SB Models}
Solving the dynamical (entropic) optimal transport is a challenging problem. \textbf{The previous SB models \citep{ipf, imf, enotdiffusion} exhibit scalability challenges when the distance between the source and target distributions is large.} One example of such a transport problem is generative modeling on an image dataset, where the source is a Gaussian distribution and the target is a high-dimensional image dataset. To address this, the previous SB models \citep{ipf, imf} rely on a pretraining strategy, such as the diffusion model with VP SDE \citep{scoresde}.
% However, \textbf{this pretraining strategy limits the primary advantage of SB models over the diffusion models}: SB models are applicable to more generalized tasks where a manually designed forward process cannot be exploited, such as Image-to-Image translation \citep{unsb}. 

In this respect, we evaluate the scalability of our model compared to existing SB models for generative modeling on CIFAR-10 \citep{cifar10}. Tab. \ref{tab:compare-cifar10} presents the results. 
The EUOT-KL refers to the EUOT problem with $D_{\Psi}$=KL divergence and the EUOT-Soft indicates the EUOT problem with $\Psi^{*} = \operatorname{Softplus}$ following \citep{uotm}. For the EOT problem,  $\Psi = \iota$, i.e., $\Psi^{*}(x) = x$ (See Appendix \ref{appen:implementation_details} for details).
All other models failed to converge without pretraining procedure, showing FID scores $\geq 100$. In contrast, \textbf{our model for the EUOT problem achieves state-of-the-art results, demonstrating comparable performance to IPF \citep{ipf} with pretraining.} Our model for the EOT problem also presents an FID score of 4.05, which is a still decent result compared to other models without pretraining. Interestingly, our EUOT model achieves better distribution matching of the target image distribution, i.e. smaller FID, than our EOT model. This result is interesting because, formally, the EUOT problem (Eq. \ref{eq:euot}) allows larger distribution errors in the target data compared to the EOT problem (Eq. \ref{eq:eot}). We interpret this as a result of the additional optimization benefit of Unbalanced Optimal Transport, as observed in \citep{uotm}. 
Moreover, \textbf{our simulation-free parametrization through the reciprocal property (Eq. \ref{eq:reciprocal_euot_model}) offers efficient training and one-step sample generation.} To be more specific, several recently proposed EOT algorithms \citep{imf, enotdiffusion} require more than 10 days of training. In contrast, our model takes approximately 4 days to train on 2 RTX 3090Ti GPUs. 

\begin{figure}
    \begin{minipage}{1\linewidth}
        \centering
        \begin{subfigure}[b]{0.48\textwidth}
            \includegraphics[width=1\textwidth]{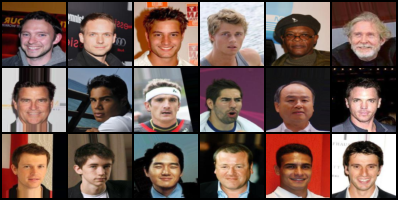}
        \end{subfigure}
        \hfill
        \begin{subfigure}[b]{0.48\textwidth}
            \includegraphics[width=1\textwidth]{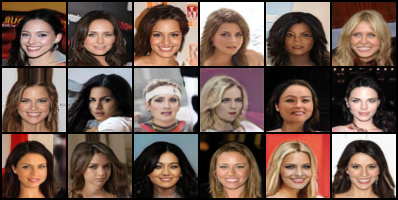}
        \end{subfigure}
        \caption{\textbf{Unpaired Male → Female translation for 64 × 64 CelebA image.}}
        \label{fig:I2I}
    \end{minipage}
    \\
    \begin{minipage}{.56\linewidth}
        \centering
        \captionof{table}{\textbf{FID score for Image-to-image Translation Tasks.}}
        \label{tab:compare-I2I}
        \centering
        \label{tab:I2I}
        \scalebox{.78}{
            \begin{tabular}{c c c}
                \toprule
                Data & Model  &  FID \\
                \midrule 
                \multirow{4}{*}{Male$\rightarrow$Female (64x64)} & DiscoGAN \citep{discoGAN} & 35.64 \\
                & CycleGAN \citep{cyclegan} & 12.94 \\
                &  NOT \citep{not} & 11.96 \\
                \cmidrule{2-3}
                & \textbf{EUOT-Soft} (Ours) & \textbf{8.44} \\
                \midrule
                \multirow{3}{*}{Wild$\rightarrow$Cat (64x64)} & DSBM \citep{imf} & 25$\le$ \\
                & \citet{sb-flow} & 25$\le$ \\
                \cmidrule{2-3}
                & \textbf{EUOT-Soft} (Ours) & \textbf{14.59} \\
                % \textbf{14.80} 
                \bottomrule
            \end{tabular}}
    \end{minipage}
    \begin{minipage}{.42\linewidth}
        \centering
        \includegraphics[width=1\textwidth]{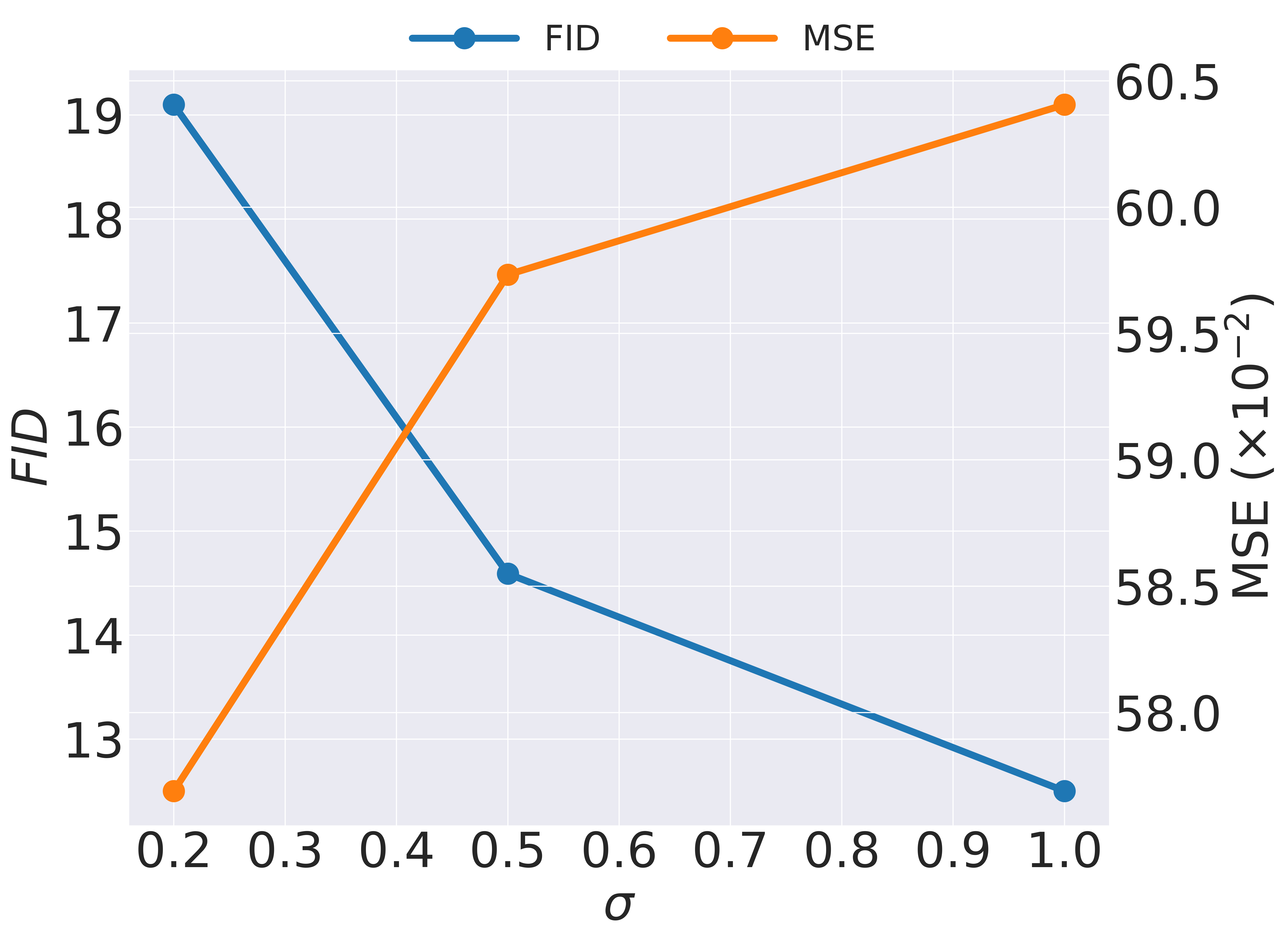}
        \vspace{-20pt}
        \caption{
        \textbf{Effect of Entropic Regularization $\sigma^2$} on Wild$\rightarrow$Cat (64x64).
        }
        \label{fig:abl_sigma}
    \end{minipage}
    \vspace{-15pt}
\end{figure}

\paragraph{Scalability Comparison on Image-to-image Translation}
We evaluated our model on multiple image-to-image (I2I) translation tasks, specifically the \textit{Male $\rightarrow$ Female} \citep{celeba} ($64 \times 64$) and \textit{Wild $\rightarrow$ Cat} \citep{afhq} ($64 \times 64$) benchmarks. Here, we tested our SF-EUOT model for the EUOT-Soft problem, which performed best in generative modeling on CIFAR-10.  Tab. \ref{tab:compare-I2I} provides the FID scores for these I2I translations tasks.
As demonstrated in Tab. \ref{tab:compare-I2I}, our model outperformed other adversarial methods, including other OT-based method of NOT \citep{not}, on the Male $\rightarrow$ Female dataset. Moreover, our model outperformed the dynamic EOT-based methods, such as DSBM \citep{imf}, on the Wild $\rightarrow$ Cat dataset.
These results demonstrate that our model achieves superior scalability in I2I translation tasks, comapred to other OT-based approaches. 
Moreover, Figure \ref{fig:I2I} illustrates the translated samples on the Male $\rightarrow$ Female dataset. Our model shows strong qualitative performance, faithfully transporting the faces and backgrounds. The additional qualitative results are provided in Appendix \ref{appen:results}.

\paragraph{Effect of Entropic Regularization $\sigma^2$}
We performed an ablation study on the entropic regularization intensity $\sigma^2$ of the EUOT problem (Eq. \ref{eq:euot}) in the image-to-image translation task on the Wild $\rightarrow$ Cat \citep{afhq} ($64 \times 64$) dataset. Here, $\sigma$ represents the noise level of the transport dynamics (Eq. \ref{eq:sto_process}). Fig. \ref{fig:abl_sigma} presents the results. As we increase $\sigma$, we observe a consistent trend of decreasing FID scores and increasing transport costs (MSE). The decrease in FID indicates that, with higher $\sigma$, our model better approximates the target distribution. The increase in transport cost (MSE) aligns with our intuition, as $\sigma$ introduces stochasticity into the model. When $\sigma$ becomes too large (e.g., $\sigma = 1$), the transported image may not align well with the identity of the source sample, leading to an increase in transport cost. Conversely, the decrease in FID can be understood as the model is more effectively mapping across diverse images.

\begin{figure}[t]
    \centering
    \begin{minipage}{1\linewidth}
        \centering
        \captionof{table}{\textbf{Comparison on Benchmarks on High Dimensional Gaussian Experiments.} $\Delta m$, $\Delta Var$, and $\Delta Cov$ stands for difference of mean, variance, and covariance, respectively. $m$, $Var$, and $Cov$ stands for ground true mean, variance and covariance, respectively.}
        \vspace{-5pt}
        \scalebox{0.8}{
        \begin{tabular}{cccccccc}    
        \toprule
        % L2 ($\times 10^2$) $\downarrow$  & $\Delta m$ & $\Delta Var$ & $\Delta Cov$ \\
        % \midrule
    %     DSB & 0.93 & 1.22 &4.87	 \\
    %     IMF-b & 0.09 &$\geq$ 10	&$\geq$ 10 \\
    %     DSBM-IPF & 0.23 &0.14 &1.01 \\
    %     DSBM-IMF & 0.35 &0.17&0.94 \\
    %     RF &	0.04 &8.29 &-	\\
    %     \midrule
    %     Ours &0.17&	2.83 & 2.84\\
    %     % \midrule
        Metric ($\%$) $\downarrow$  &  DSB & IMF-b & DSBM-IPF & DSBM-IMF & RF & Ours \\
        \midrule
        $\Delta m / m$     & 9.3   & 0.9	    & 2.3	& 3.5   &0.4	& 2.9 \\
        $\Delta Var / Var$ & 1.22	&$\geq$ 10	&0.14	&0.17	&8.29   &	3.00 \\
        $\Delta Cov / Cov$ & 7.88	&$\geq$ 10	& 1.63	& 1.52	&-	    & 3.72 \\
       \bottomrule
    \end{tabular}}
    \label{tab:high-Gaussian}
    \end{minipage}
    \\
    \vspace{10pt}
    \begin{minipage}{1\linewidth}
        \captionsetup[subfigure]{aboveskip=0pt,belowskip=-1pt}
        \centering
        \begin{subfigure}[b]{0.24\textwidth}
            \includegraphics[width=1\textwidth]{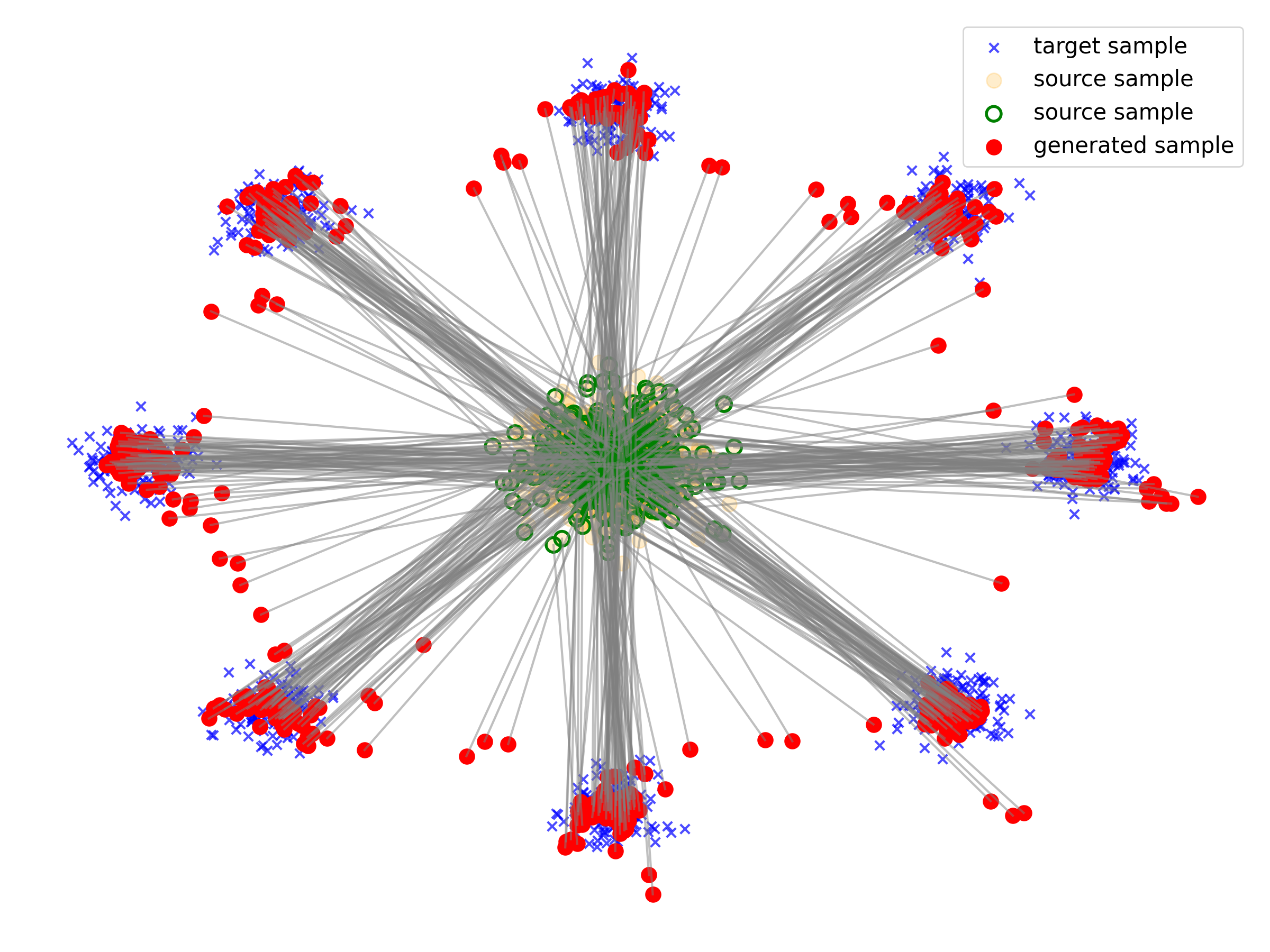}
            \caption{(G $\rightarrow$ 8G) - Ours}
            \label{fig:generated_8gaussian}
        \end{subfigure}
        \begin{subfigure}[b]{0.24\textwidth}
            \includegraphics[width=1\textwidth]{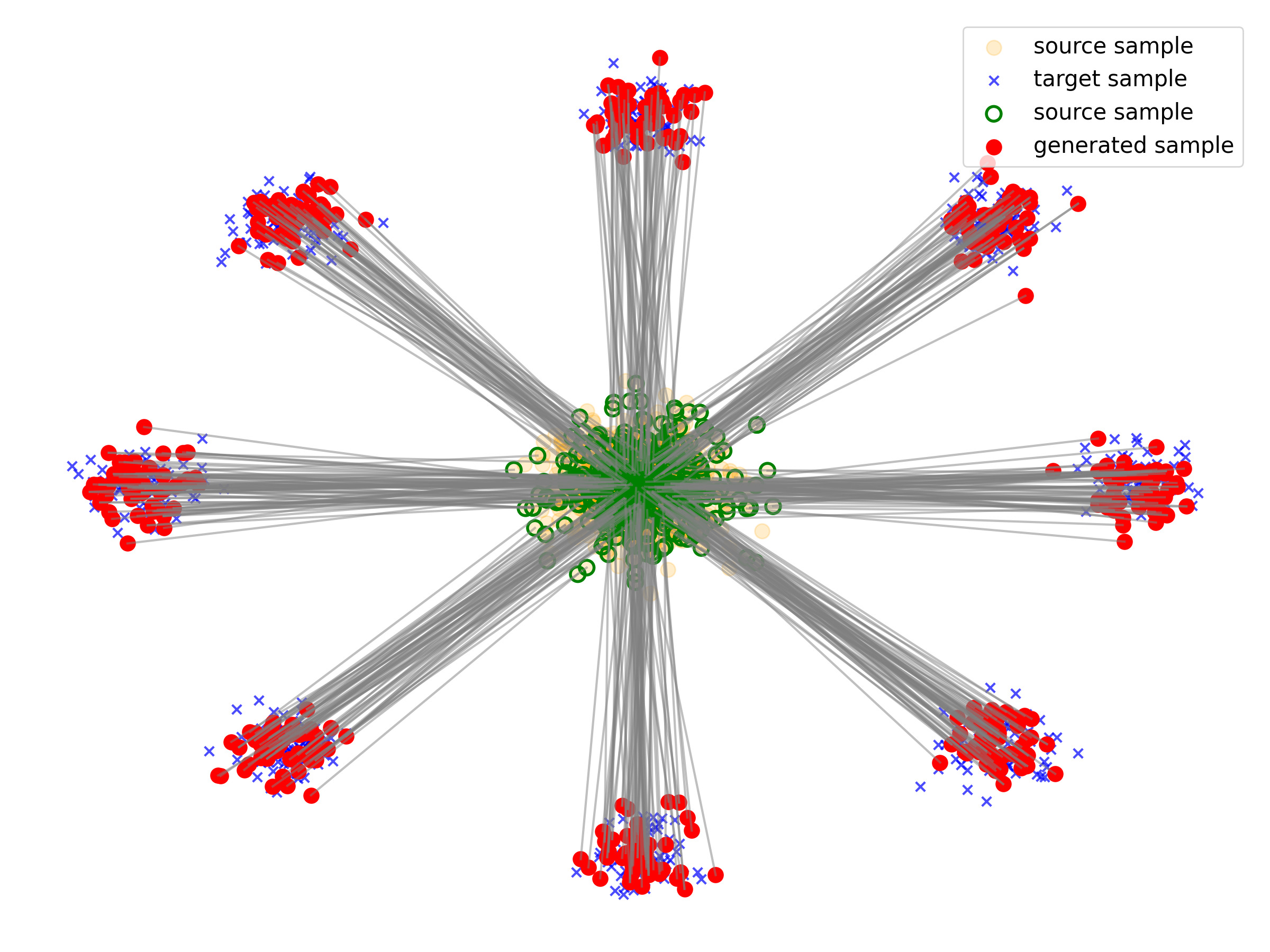}
            \caption{(G $\rightarrow$ 8G) - GT}
            \label{fig:real_8gaussian}
        \end{subfigure}
        
        \centering
        \begin{subfigure}[b]{0.24\textwidth}
            \includegraphics[width=1\textwidth]{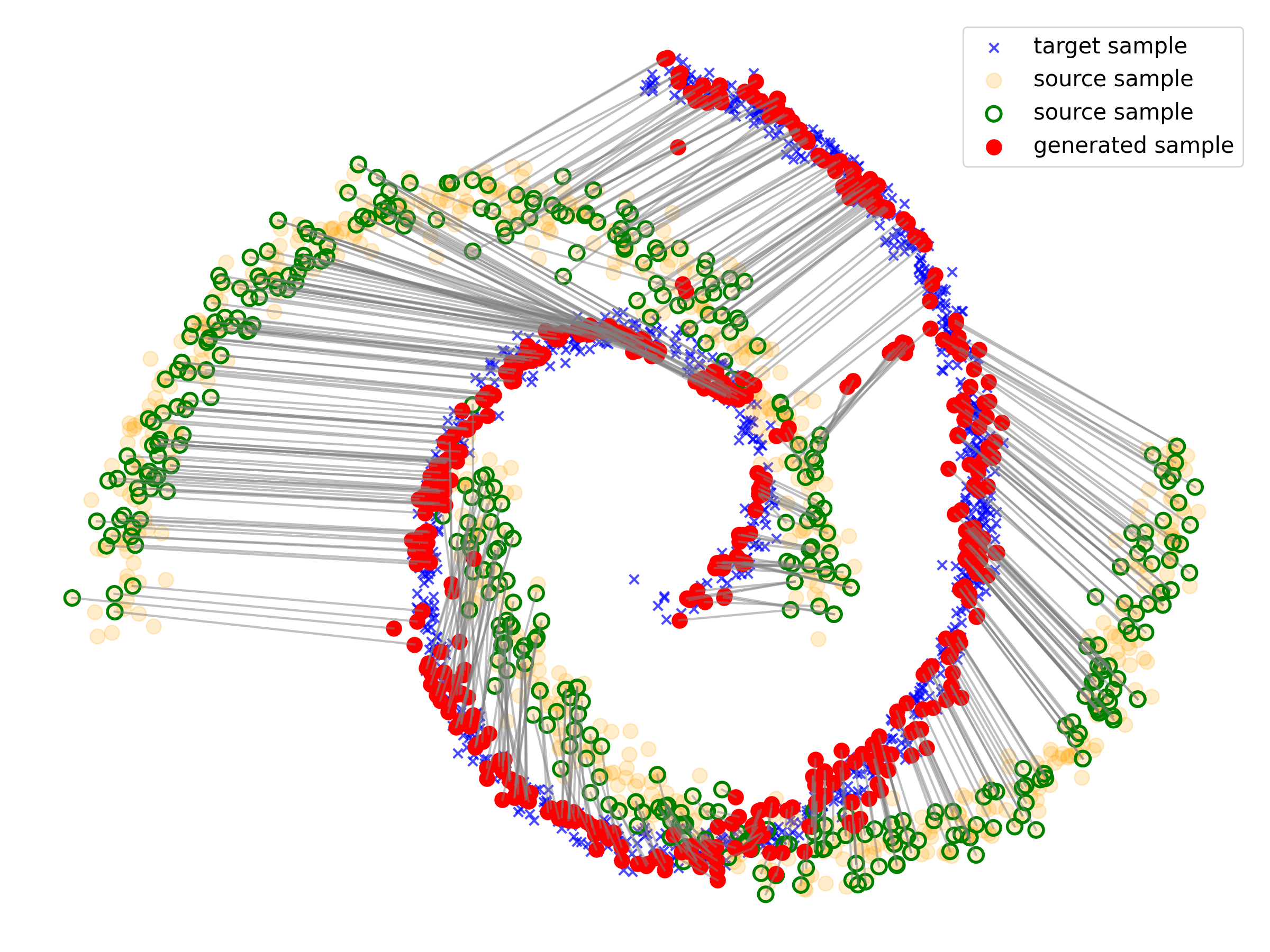}
            % \caption{Moon $\rightarrow$ Spiral - Ours}
            \caption{(M $\rightarrow$ S) - Ours}
            \label{fig:generated_moon2spiral}
        \end{subfigure}
        \begin{subfigure}[b]{0.24\textwidth}
            \includegraphics[width=1\textwidth]{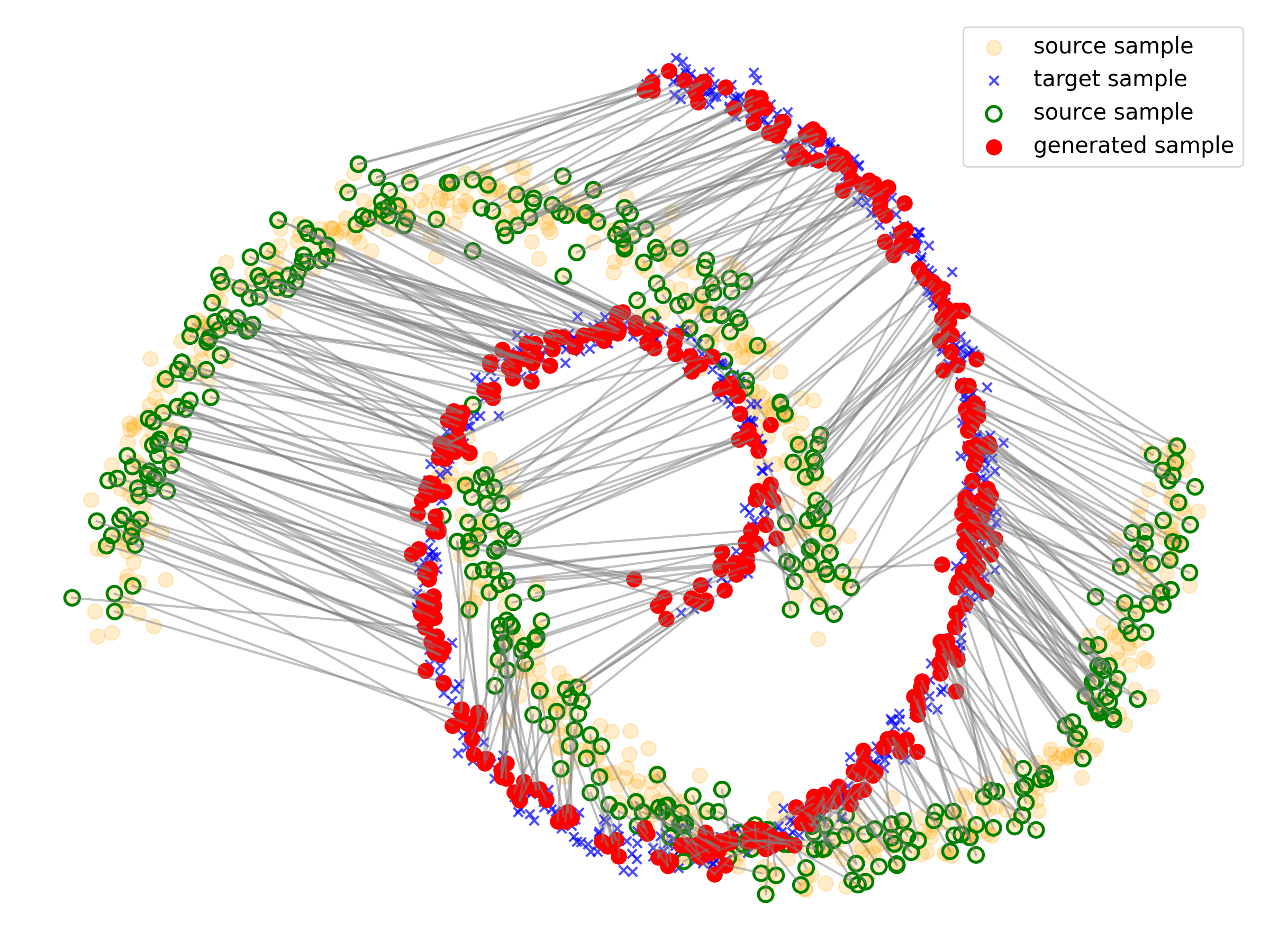}
            % \caption{Moon $\rightarrow$ Spiral - GT}
            \caption{(M $\rightarrow$ S) - GT}
            \label{fig:real_moon2spiral}
        \end{subfigure}
        \vspace{-5pt}
      \caption{\textbf{Comparison of EOT Transport Plan $\boldsymbol{\pi}$ between Our model and Discrete Ground-Truth} from POT \citep{pot} when $\sigma = 0.75$. The gray lines illustrate the generated pairs, i.e., the connecting lines between $x$ (green) and $\pi(y | x)$ (red).
      }
        \label{fig:2d}
    \end{minipage}
    \vspace{-15pt}
\end{figure}

\subsection{Comparison to EOT Solution on Synthetic Dataset} \label{sec:exp_toy}
\vspace{-3pt}
\paragraph{Qualitative Results on 2D Toy Datasets} We evaluate whether \textbf{our model can learn the ground-truth solution of the EOT (EUOT with $\Psi=\iota$) problem.} Specifically, we compare the trained static coupling $\pi_{\theta}$ (Eq. \ref{eq:euot}) with the proxy ground-truth coupling obtained using the convex OT solver in the POT library \citep{pot}. 
Note that the POT library provides the solution between two \textit{empirical discrete measures} derived from the training data, while our goal is to solve the EUOT problem between two \textit{continuous measures} $\mu, \nu$. 
% We evaluate our model on EOT instead of EUOT because, to the best of our knowledge, there is no closed-form solution for EUOT in the general case.

Fig. \ref{fig:2d} presents the results on two datasets: \textit{Gaussian to 8-Gaussian (G $\rightarrow$ 8G)} and \textit{Moon to Spiral (M $\rightarrow$ S)}. (See Appendix \ref{appen:implementation_details} for implementation details). 
As shown in Fig. \ref{fig:2d}, our model exhibits a decent performance in learning the EUOT transport plan, successfully generating the target distribution and providing a rough approximation of the optimal transport map for individual source samples. However, some noisy outputs are observed, and the model does not fully approximate a precisely optimal transport plan.
We interpret this phenomenon in terms of the difficulty in training through PDE-like learning objectives \citep{pinn, pinn2}. In this respect, we believe there is room for development in imposing optimal conditions for the path measure $\rho$. 
Nevertheless, Sec. \ref{sec:exp_image} demonstrates that our model presents considerable scalability enhancement over previous models.

\vspace{-5pt}
\paragraph{Quantitative Results on High-Dimensional Gaussian Dataset}
We conduct a quantitative evaluation to test whether our model can accurately learn the optimal coupling. For this evaluation, we exploit the closed form solution for the Entropic Optimal Transport (EOT) problem in the Gaussian-to-Gaussian case \citep{sb_gtsol}. There are two reasons for conduting evaluation on the EOT problem, not on the EUOT problem. First, the EOT problem between Gaussians has a closed-form solution for the optimal transport coupling, while there is no such closed form solution for the EUOT problem. Second, several methods for the Schr\"{o}dinger Bridge problem, i.e., the dynamic EOT problem, conducted same benchmark, allowing us to provide broader comparisons with other methods. We follow the experimental settings and evaluation metric of \citep{imf}.

Tab. \ref{tab:high-Gaussian} shows the quantitative  evaluation results in terms for three metrics: the relative error of the mean $\Delta m / m$, variance $\Delta Var / Var$, and covariance $\Delta Cov / Cov$ compared to the ground-truth mean, variance, and covariance, respectively. Note that the relative errors for the mean and variance assess the distribution error between the generated distribution $T_\#\mu$ and the ground-truth target measure $\nu$. The relative error of covariance evaluates whether our model learned optimal coupling. Overall, our model shows comparable or slightly worse performance than other models. However, since the relative error is below 4\% in all three metrics, our model can be considered to exhibit moderate performance. We hypothesize that the adversarial training method and minimizing the PDE-like objective for $T_\theta$ resulted in this rough approximation of the optimal coupling. Nevertheless, we would like to emphasize that the strength of our methodology lies in developing a scalable algorithm that transports a source sample into the target with a single-step evaluation. In Sec. \ref{sec:exp_image}, we discuss further the scalability and efficiency of our proposed method.

\vspace{-4pt}
\section{Conclusion}
\vspace{-5pt}
In this paper, we propose an algorithm for solving the Entropic Unbalanced Optimal Transport (EUOT) problem. We derived the dynamical formulation of EUOT. Then, we established the dual form and analyzed this dual form from the stochastic optimal control perspective. 
Our model is based on the simulation-free algorithm leveraging the reciprocal property of the dynamical formulation of EUOT problem. 
Our experiments demonstrated that our model addresses the scalability challenges of previous Schrödinger Bridge models. Specifically, our model offers simulation-free training and achieves state-of-the-art results in generative modeling on CIFAR-10 without diffusion model pretraining.
A limitation of this work is that our method demonstrates lower accuracy in learning the EUOT compared to other models. We hypothesize that this is due to the inherent difficulty of achieving precise matching using a PINN-style loss function. Additionally, due to computational resource constraints, we were unable to test our model on high-resolution datasets such as CelebA-HQ \citep{celeba} ($256 \times 256$).

\section*{Ethics Statement} % \label{appen:broader_impacts}
Our approach significantly enhances the scalability of EOT algorithms, enabling the generation of high-quality samples from large-scale datasets while maintaining an accurate representation of the data distribution. As a result, we expect our model to impact various fields, including image transfer, finance, image synthesis, healthcare, and anomaly detection.
However, it is important to recognize the potential negative societal implications of our work. Generative models can unintentionally learn and magnify existing biases within the data, which may reinforce societal biases. Therefore, careful monitoring and control are crucial when deploying these models in real-world applications. Rigorous management of both the training data and the modeling process is essential to mitigate any potential negative societal effects.

\section*{Reproducibility Statement}
To ensure the reproducibility of our work, we submitted the anonymized source in the supplementary material, provided complete proofs of our theoretical results in Appendix \ref{appen:proof_derivation}, and included the implementation and experiment details in Appendix \ref{appen:implementation_details}.

\newpage
\bibliography{iclr2025_conference}
\bibliographystyle{iclr2025_conference}

\newpage
\appendix

\section{Proofs and Derivations} \label{appen:proof_derivation}
In this section, we provide the proof of the theorems in Sec. \ref{sec:dynamic_dual_euot}.
\jm{Moreover, we introduce another dual form and its relationship to our work. Furthermore, we derive the optimization problem Eq. \ref{eq:dual_interpret4} and justify the conditional sampling in line 5 of Algorithm \ref{alg:em_euot}.}
For all theorems, we assume that $\Psi$ is an \jm{differentiable} entropy function that satisfies \jm{superlinearity, i.e.} $\Psi^{'}_\infty := \lim_{x\rightarrow \infty} \Psi(x)/ x = \infty$.
In this case, $D_\Psi(\rho_1 | \nu)$ is infinity whenever $\rho_1$ has singularity with respect to $\nu$.
\jm{Thanks to the superlinearity, continuity and convexity of $\Psi$, $D_{\Psi}$ is a lower semi-conitinuous function.}
% Moreover, except the case when $\Psi$ is a convex indicator, we only consider differentiable $\Psi$.

\subsection{Proofs} \label{appen:proofs}
The following lemma implies that the search space of the joint distribution $\pi$ in the EUOT problem can be extended to unnormalized density space $\mathcal{M}_2$.
Based on this lemma, we abuse the notation for the search space in the following theorems.

\begin{lemma} \label{theorem:lemma}
    Let $\pi^\star$ be the optimal plan  for
    \begin{equation} \label{eq:extended_euot_appen}
        \inf_{\pi_0 = \mu, \pi \in \mathcal{M}_2 (\mathcal{X}\times \mathcal{X})} \left[ \int_{\mathcal{X}\times \mathcal{X}} \frac{1}{2} \lVert x - y \rVert^2 d\pi (x,y) - \sigma^2 H(\pi) + \alpha D_\Psi (\pi_1 | \nu) \right].
    \end{equation}
    Note that the search space of $\pi$ is extended to unnormalized density space $\mathcal{M}_2$ instead of using $\mathcal{P}_2$ as in EUOT problem defined in Eq. \ref{eq:euot}.
    Even if the search space is extended, the mass of the optimal target marginal $\pi^\star_1$ is 1.
    In other words, the problem Eq. \ref{eq:extended_euot_appen} is equivalent to Eq. \ref{eq:euot}.
\end{lemma}
\begin{proof}
    The well-known dual form of the Eq. \ref{eq:extended_euot_appen} \cite{eot_dual} is defined as follows:
    \begin{equation} \label{eq:euot_standard_dual_appen}
        \sup_{u, v} \int_{\mathcal{X}} u(x) d\mu(x) - \int_{\mathcal{X}} \Psi^*(-v(y)) d\nu(y) - \epsilon \int_{\mathcal{X}\times \mathcal{X}} e^{\frac{u(x) + v(y) -c(x,y)}{\epsilon} } d\mu(x) d\nu(y).
    \end{equation}
    Thanks to the Fenchel-Rockafellar theorem \cite{fenchel-rockafellar}, the strong duality holds (See Proposition 4.2 in \cite{uot-robust}.
    The first variation of Eq. \ref{eq:euot_standard_dual_appen} with respect to the pair of the optimal potentials $(u^\star, v^\star)$ is as follows:
    \begin{equation} \label{eq:first_variation_appen}
        \int_{\mathcal{X}} \delta u(x) d\mu(x) + \int_{\mathcal{X}} \delta v {\Psi^*}' (-v^\star(y)) d\nu(y) - \int_{\mathcal{X}\times \mathcal{X}} \left(\delta u + \delta v \right) e^{\frac{u^\star(x) + v^\star(y) -c(x,y)}{\epsilon} } d\mu(x) d\nu(y).
    \end{equation}
    Now, let $ \Tilde{\nu}(y) = {\Psi^*}'(-v^\star (y)) \nu(y)$.
    If the $(\delta u, \delta v) = (\lambda, -\lambda)$, then the Eq. \ref{eq:first_variation_appen} can be written as follows: 
    \begin{equation}
        \int \lambda d\mu - \int \lambda d\Tilde{\nu} = \lambda (1 - {\rm m}(\Tilde{\nu})),
    \end{equation}
    where ${\rm m}(\cdot)$ denotes the mass of the measure.
    Since the potentials are optimal, the mass of $\Tilde{\nu}$ should be 1.
    Furthermore, reordering the first variation with respect to $\delta v$ in Eq. \ref{eq:first_variation_appen}, we can derive that
    \begin{equation}
        {\Psi^*}' (-v^\star(y)) = \int e^{\frac{u^\star(x) + v^\star(y) -c(x,y)}{\epsilon}} d\mu(x).
    \end{equation}
    Then, by leveraging the primal-dual relationship \cite{uot-robust}, i.e. 
    \begin{equation}
        d \pi^\star(x,y) = e^{\frac{u^\star(x)+v^\star(y) -c(x.y)}{\epsilon}}  d\mu(x) d\nu(y),
    \end{equation}
    we can derive the following equation:
    \begin{equation}
        \pi^\star_1(y) = \int \pi^\star(x,y) dx = \left( \int   e^{\frac{u^\star(x)+v^\star(y) -c(x.y)}{\epsilon}} d\mu(x) \right) \nu(y) = {\Psi^*}'(-v^\star(y)) \nu(y) = \Tilde{\nu}(y).
    \end{equation}
    Since ${\rm m}(\Tilde{\nu}) = 1$, $\pi^\star_1$ has a mass of 1.
\end{proof}

\begin{theorem} \label{theorem:dynamic_EUOT_appen}
    The EUOT problem is equivalent to the following dynamic transport problem:
    \begin{equation} \label{eq:dynamic_euot_appen}
        \inf_u \left[ \int_0^1 \int_{\mathcal{X}} \frac{1}{2}\| u(t,x) \|^2 d\rho_t(x) dt + \alpha D_\Psi (\rho_1| \nu)\right],
    \end{equation}
    where $\partial_t \rho_t + \nabla \cdot (u \rho_t ) - \frac{\sigma^2}{2} \Delta \rho_t = 0$ and $\rho_0 = \mu$.
    Moreover, the optimal solution $\mathbb{P}^\star$ satisfies the reciprocal property, i.e., 
    \begin{equation} \label{eq:reciprocal_euot_appen}
        \mathbb{P}^{\star}_t (\cdot|x,y) = \mathcal{N}(\cdot| (1-t)x+ ty, \sigma^2 t (1-t) I), \quad (x,y) \sim \mathbb{P}^\star_{0,1} \ \text{almost surely.}
    \end{equation}
\end{theorem}
\begin{proof}
    Suppose $u^\star$ is the solution of Eq. \ref{eq:dynamic_euot_appen}. 
    Let $\mathbb{P}^\star$ be the path measure induced by $X^{u^\star}$.
    Since $\Psi^{'}_\infty = \infty$, $\mathbb{P}^\star_1 (= \rho_1)$ is absolutely continuous with respect to $\nu$. 
    Note that $\mathbb{P}^\star_1 \in \mathcal{P}_2(\mathcal{X})$ by the Lemma \ref{theorem:lemma}.
    This implies that the objective of $u$ in Eq. \ref{eq:dynamic_euot_appen} is to solve the SB problem between $\mu$ and $\mathbb{P}^\star_1$.
    Thus, $\mathbb{P}^\star_t$ satisfies the reciprocal property.
    Moreover, using the reciprocal property, Eq. \ref{eq:dynamic_euot_appen} could be reformulated to the static formulation:
    \begin{equation} \label{eq:klplusdiv_euot_appen}
        \inf_{\pi\in \mathcal{P}_2, \ \pi_0 = \mu} \left[ \sigma^2 D_{\text{KL}}(\pi| \mathbb{Q}_{0,1}) + \alpha D_{\Psi}(\pi_1 | \nu) \right].
    \end{equation}
    Now, by applying Eq. \ref{eq:def_kl}, we obtain
    \begin{equation} \label{eq:euot_appen}
        \inf_{\pi_0 = \mu, \pi \in \mathcal{P}_2 (\mathcal{X})} \left[ \int_{\mathcal{X}\times \mathcal{X}} \frac{1}{2} \lVert x - y \rVert^2 d\pi (x,y) - \sigma^2 H(\pi) + \alpha D_\Psi (\pi_1 | \nu) \right].
    \end{equation}  
\end{proof}

\begin{proposition}[\textbf{Dual I}]
    The EUOT problem is equivalent to the following problem:
    \begin{equation} \label{eq:weakot_euot_appen}
            \sup_{V\in \Phi_{2,b}} \inf_{u} \left[ \mathbb{E} \left[\int_0^1 \frac{1}{2} \| u(t,X^u_t) \|^2 dt + V(X^u_1) \right] -  \int_{\mathcal{X}} \alpha \Psi^* \left( \frac{V(y)}{\alpha} \right) d\nu(y) \right], \quad {\rm s.t.} \ X_0 \sim \mu.
    \end{equation}
\end{proposition}
\begin{proof}
    % Let $E := \{\pi \in \mathcal{P}_2(\mathcal{X}\times \mathcal{X}) : \pi_0 = \mu \}$, $F(\pi): = \sigma^2 D_{\text{KL}}(\pi | \mathbb{Q}_{0,1})$ and $G(\pi) := D_\Psi(\pi_1 | \nu)$.
    \jm{Let $F(\pi): = \sigma^2 D_{\text{KL}}(\pi | \mathbb{Q}_{0,1})$ and $G(\pi) := D_\Psi(\pi_1 | \nu)$.}
    Then, Eq. \ref{eq:klplusdiv_euot_appen} can be rewritten as follows:
    \begin{equation}
        \inf_{\pi \in \jm{\mathcal{P}_2}} \left[ {F(\pi) + G(\pi)}\right].
    \end{equation}
    \jm{Note that $F$ and $G$ are convex lower semi-continuous functions. Thus, by applying Fenchel-Rockafellar theorem \cite{fenchel-rockafellar}, we obtain the following duality form:}
    % Note that $F$ and $G$ are convex functions and $F$ $G$ is continuous on the space $S := \{\pi\in E: \pi_1 \ll \nu \}$.
    % Thus, by applying Fenchel-Rockafellar theorem \cite{fenchel-rockafellar} on the space $S$ we obtain the following duality form:
    \begin{equation}
        \inf_{\pi \in \jm{\mathcal{P}_2}} {F(\pi) + G(\pi)} = \sup_{V \in \jm{\mathcal{P}^*_2}}\left[ -F^*(-V) - G^*(V) \right].
    \end{equation}
    Note that $\jm{\mathcal{P}^*_2} = \Phi_{2,b}$ by Lemma 9.8 in \cite{proof1}.
    Moreover, by proof of Theorem 9.5 in \cite{proof1}, 
    \begin{equation}
        -F^*(-V) = \inf_{\pi \in P_2(\mathcal{X})} \left[ \sigma^2D_{\text{KL}}(\pi | \mathbb{Q}_{0,1}) + \int V(y) d \pi_1(y) \right].
    \end{equation}
    Since $G^*(V) = \int \alpha \Psi^*(V(y) / \alpha) d\nu(y)$, we finally obtain the following dual form:
    \begin{equation} \label{eq:klplusdual_appen}
        \sup_{V\in \Phi_{2,b}} \inf_{\pi \in \mathcal{P}_2} \left[ \sigma^2 D_{\text{KL}} (\pi |\mathbb{Q}_{0,1}) + \int V(y) d\pi_1(y)  -  \int_{\mathcal{X}} \alpha \Psi^* \left( \frac{V(y)}{\alpha} \right) d\nu(y) \right].
    \end{equation}
    Through the discussion in the proof of Theorem \ref{theorem:dynamic_EUOT_appen} or Sec. 2.2 in \citet{enotdiffusion}, $\pi$ can be replaced by the distribution of the stochastic process $\{X^u_t\}$ and $D_{\text{KL}} (\pi | \mathbb{Q}_{0,1}) = \mathbb{E}\left[ \int \lVert u_t \rVert^2 / (2\sigma^2) dt \right]$.
    By the replacement, we obtain the following dual form of EUOT.
    \begin{equation}
            \sup_{V\in \Phi_{2,b}} \inf_{u} \left[ \mathbb{E} \left[\int_0^1 \frac{1}{2} \| u(t,X^u_t) \|^2 dt + V(X^u_1) \right] -  \int_{\mathcal{X}} \alpha \Psi^* \left( \frac{V(y)}{\alpha} \right) d\nu(y) \right],
    \end{equation}
    where $X_0 \sim \mu$.
\end{proof}

\subsection{Derivations} \label{appen:derivations}
\jm{We provide the another dual formulation of dynamical EUOT. Then, we also introduce the connection to our dual form, i.e. Eq. \ref{eq:weakot_euot}. Furthermore, we derive the optimization problem Eq. \ref{eq:dual_interpret4}. Finally, we provide the justification of conditional sampling in line 5 of Algorithm \ref{alg:em_euot}.}

\paragraph{Lagrangian Dual of Dynamical Form of EUOT}
Starting from the dynamical formulation of EUOT (Eq. \ref{eq:dynamic_euot_appen}), we can also derive an alternative dual formulation of EUOT. Note that when $\Psi$ is a convex indicator, the following Eq. \ref{eq:lagrangian_euot_appen} corresponds to the SB objective of \citet{wlf} (Example 4.4). Specifically, the dual formulation is expressed as follows:
\begin{proposition}[\textbf{Dual II}]\label{theorem:lagrangian_euot_appen}
    The dual form of Eq. \ref{eq:dynamic_euot_appen} is written as follows:
    \begin{equation} \label{eq:lagrangian_euot_appen}
            \sup_{A} \inf_{\rho}  \left[ \int_0^1 \int_{\mathcal{X}} \left( \partial_t A_t -\frac{1}{2} \| \nabla A_t \|^2 + \frac{\sigma^2}{2} \Delta A_t \right) d\rho_t dt + \int_{\mathcal{X}} A_0 d\mu  - \int_{\mathcal{X}} \alpha \Psi^* \left(\frac{A_1}{\alpha}\right) d\nu \right],
    \end{equation}
    where $\rho_0 = \mu$ and $A:[0,1]\times \mathcal{X} \rightarrow \mathbb{R}$.
    Furthermore, for the optimal $u^\star$ in Eq. \ref{eq:dynamic_euot_appen} and optimal $A^\star$, $u^\star = - \nabla A^\star$.
\end{proposition}
\begin{proof}
    The dual form of Eq. \ref{eq:dynamic_euot_appen} is
    \begin{align} \label{eq:dual1_lagrangian_euot_appen}
        \begin{split}
            \inf_{\rho, u} \sup_{A, \lambda_0, \lambda_1} \ \  &\int_0^1 \int_\mathcal{X} \frac{1}{2} \lVert u(t,x) \rVert^2 \rho_t(x) dx dt + \alpha \int_{\mathcal{X}} \Psi\left(\frac{dP(x)}{d\nu(x)}\right) d\nu(x) \\
            - &\int^1_0 \int_{\mathcal{X}} A(t,x) \left( \partial_t \rho_t(x) + \nabla \cdot (u \rho_t (x)) - \frac{\sigma^2}{2} \Delta \rho_t (x) \right) dx dt \\
            + &\int_{\mathcal{X}} \lambda_0(x) (\rho_0(x) - \mu(x)) dx + \int_{\mathcal{X}} \lambda_1(x) (\rho_1(x) - P(x)) dx,
        \end{split}
    \end{align}
    where $A:[0,1]\times \mathcal{X} \rightarrow \mathbb{R}$ and $\lambda_0, \lambda_1 : \mathcal{X} \rightarrow \mathbb{R}$ are the dual variables.
    Note that we can freely swap the order of the optimization since the optimization problem is convex in $\rho,\ P,\ u$, and linear in $A,\  \lambda_0, \ \lambda_1$.
    By applying integration by parts to second line of Eq. \ref{eq:dual1_lagrangian_euot_appen}, we obtain
    \begin{align} \label{eq:dual2_lagrangian_euot_appen}
        \begin{split}
            \sup_{A, \lambda_0, \lambda_1} \inf_{\rho, P, u} \ \  &\int_0^1 \int_\mathcal{X} \frac{1}{2} \lVert u \rVert^2 d\rho_t dt + \alpha \int_{\mathcal{X}} \Psi\left(\frac{dP}{d\nu}\right) d\nu \\
            + &\int^1_0 \int_{\mathcal{X}} \left( \partial_t A_t + \langle u, \nabla A_t \rangle + \frac{\sigma^2}{2} \Delta A_t \right) d\rho_t dt - \int_{\mathcal{X}} A_1 d\rho_1 + \int_{\mathcal{X}} A_0 d\rho_0 \\
            + &\int_{\mathcal{X}} \lambda_0(x) (\rho_0(x) - \mu(x)) dx + \int_{\mathcal{X}} \lambda_1(x) (\rho_1(x) - P(x)) dx.
        \end{split}
    \end{align}
    Here, note that the last two terms in the second line of Eq. \ref{eq:dual1_lagrangian_euot_appen} pop out from the integration of parts with respect to the time variable $t$.
    Since the above problem with respect to $u$ is quadratic, we can obtain the explicit solution $u = - \nabla A$.
    Moreover, by freely swapping the optimization variables, we obtain $\lambda_0 = A_0$, $\lambda_1 = A_1$, $\rho_0 = \mu$, and $\rho_1 = P$.
    Moreover, by the definition of convex conjugate, we obtain Eq. \ref{eq:lagrangian_euot_appen}:
    \begin{align} \label{eq:dual3_lagrangian_euot_appen}
        \begin{split}
            &\sup_A  \inf_{\rho} \ \int^1_0 \int_{\mathcal{X}}  \partial_t A_t - \frac{1}{2} \lVert \nabla A_t \rVert^2 + \frac{\sigma^2}{2} \Delta A_t \ d\rho_t dt  \\
            & \qquad \quad -  \int_{\mathcal{X}} A_1 \frac{d\rho_1}{d\nu} - \alpha \Psi\left(\frac{d\rho_1}{d\nu}\right) d\nu + \int_{\mathcal{X}} A_0 d\mu,
        \end{split}
        \\
        = &\inf_{\rho} \sup_A \int^1_0 \int_{\mathcal{X}} \left( \partial_t A_t - \frac{1}{2} \lVert \nabla A_t \rVert^2 + \frac{\sigma^2}{2} \Delta A_t \right) d\rho_t dt - \int_{\mathcal{X}} \alpha \Psi^* \left(\frac{A_1}{\alpha}\right) d\nu + \int_{\mathcal{X}} A_0 d\mu.
    \end{align}
\end{proof}

\paragraph{Relationship between Two Dual Formulations}
We establish the equivalence between the two dual forms: \textit{Dual I} and \textit{Dual II}.
\begin{proposition} \label{prop:dualisdual_appen}
    The Lagrangian dual formulation of Eq. \ref{eq:weakot_euot_appen} (Dual I) is equivalent to Eq. \ref{eq:lagrangian_euot_appen} (Dual II).
    Moreover, let $V^\star:\mathcal{X} \rightarrow \mathbb{R}$ be the solution of $V$ in Eq. \ref{eq:weakot_euot_appen} and let $A^\star:[0,1]\times \mathcal{X} \rightarrow \mathbb{R}$ and $\rho^\star$ be the solution of $A$ and $\rho$ in Eq. \ref{eq:lagrangian_euot_appen}, respectively.
    Then, $V^\star(x) = A^\star_1(x)$ in $\rho^\star_1$ almost surely.
\end{proposition}
\begin{proof}
    We derive the results by following the proof of Proposition \ref{theorem:lagrangian_euot_appen}.
    The Lagrangian dual of Eq. \ref{eq:weakot_euot_appen} can be reformulated as follows:
    \begin{align}
        \begin{split}
            &\sup_{V,A} \inf_{u,\rho} \ \int_0^1 \int_{\mathcal{X}} \frac{1}{2} \| u \|^2 d\rho_t  dt + \int_{\mathcal{X}} V_1 d\rho_1 - \int_{\mathcal{X}} \alpha \Psi^* \left( \frac{V_1}{\alpha} \right) d\nu \\
            & \qquad \qquad \quad -\int_0^1 \int_{\mathcal{X}} A \left[ \partial_t \rho + \nabla \cdot (u \rho) - \frac{\sigma^2}{2} \Delta \rho \right]dxdt,
        \end{split}
        \\
        \begin{split}
            = &\sup_{V,A} \inf_{\rho \in \mathcal{M}_2} \ \int_{\mathcal{X}} V_1-A_1 d\rho_1 + \int_{\mathcal{X}} A_0 d\rho_0 - \int_{\mathcal{X}} \alpha \Psi^* \left( \frac{V_1}{\alpha} \right) d\nu \\
            & \qquad \qquad \quad + \int_0^1 \int_{\mathcal{X}} \left( \partial_t A_t - \frac{1}{2} \| A_t \|^2 + \frac{\sigma^2}{2} \Delta A_t \right) d\rho_t dt,
        \end{split}
    \end{align}
    where $\rho_0 = \mu$.
    Note that $\rho$ is a non-negative Borel measure.
    If $V_1(x) < A_1(x)$ for some $x$, then the infimum of the first term of the above equation with respect to $\rho_1$ becomes $-\infty$.
    Thus, $V_1 \geq A_1$.
    Moreover, whenever $V_1 > A_1$, the corresponding $\rho_1$ vanishes.
    Thus, $V^\star_1 = A^\star_1$ for $\rho^\star_1$-almost surely.
    Therefore, by the optimality condition, the problem boils down to the following optimization problem:
    \begin{align}
        \begin{split}
            \sup_{A}  \inf_{\rho} \ \int_{\mathcal{X}} A_0 d\rho_0 - \int_{\mathcal{X}} \alpha \Psi^* \left( \frac{A_1}{\alpha} \right) d\nu + \int_0^1 \int_{\mathcal{X}} \left( \partial_t A_t - \frac{1}{2} \| A_t \|^2 + \frac{\sigma^2}{2} \Delta A_t \right) d\rho_t dt.
        \end{split}
    \end{align}
\end{proof}

\paragraph{Derivation of Inner-loop Objective (Eq. \ref{eq:dual_interpret4})}
The inner-loop problem of Eq. \ref{eq:dual_interpret2} can be written as follows:
\begin{align} \label{eq:inner_loop_appen}
    \begin{split}
    \inf_{(u, \rho)} &\mathbb{E}_{(t,x)\sim \rho} \left[ \frac{1}{2} \lVert u(t,x) \rVert^2 \right] + \mathbb{E}_{\hat{y}\sim \rho_1}\left[ V(1,\hat{y}) \right], \\
    {\text{s.t.   }}&\partial_t \rho + \nabla \cdot (u\rho) - (\sigma^2 / 2) \Delta \rho = 0, \ \rho_0 = \mu.
    \end{split}
\end{align}
Then, the dual form of Eq. \ref{eq:inner_loop_appen} with the dual variable $\lambda:[0,1]\times \mathcal{X}\rightarrow \mathbb{R}$ can be easily derived as follows:
\begin{multline} \label{eq:inner_loop_induce1_appen}
    \sup_\lambda \inf_{(u, \rho)} \int_{\mathcal{X}}\int^1_0 \frac{1}{2} \lVert u(t,x)\rVert^2 \rho(t,x) dt dx + \int_{\mathcal{X}} V(1,\hat{y}) \rho(1,\hat{y}) \\ - \int_{\mathcal{X}}\int^1_0 \lambda (t,x) \left(  \partial_t \rho(t,x) + \nabla \cdot (u\rho) - \frac{\sigma^2}{2} \Delta \rho \right) dtdx,
\end{multline}
where $\rho_0 = \mu$. From now on, we omit the condition $\rho_0 = \mu$ for simplicity.
By applying integration by parts to the last term of Eq. \ref{eq:inner_loop_induce1_appen}, we obtain
\begin{multline} \label{eq:inner_loop_induce2_appen}
    \sup_\lambda \inf_{(u, \rho)} \int_{\mathcal{X}}\int^1_0 \frac{1}{2} \lVert u(t,x)\rVert^2 \rho(t,x) dt dx + \int_{\mathcal{X}} \left( V(1,\hat{y}) - \lambda(1, \hat{y}) \right) \rho(1,\hat{y}) + \int_{\mathcal{X}} \lambda(0,x) d\mu(x) \\ 
    + \int_{\mathcal{X}}\int^1_0 \left( \partial_t \lambda (t,x) + \nabla \lambda(t,x) \cdot u (t,x)-\frac{\sigma^2}{2} \Delta \lambda(t,x) \right)\rho(t,x)  dtdx.
\end{multline}
    
Since the minimization objective of Eq. \ref{eq:inner_loop_induce2_appen} is quadratic with respect to variable $u$, we can derive the explicit solution $u=-\nabla \lambda$. Thus, we obtain
\begin{equation} \label{eq:inner_loop_induce3_appen}
    \sup_{\lambda} \inf_{\rho} \left[ \int^1_0 \int_{\mathcal{X}} \left(\partial_t \lambda - \frac{1}{2} \lVert \nabla \lambda \rVert^2 + \frac{\sigma^2}{2} \Delta \lambda \right) \rho_t dx dt + \int_{\mathcal{X}} (V_1 - \lambda_1) \rho_1 dx + \int_{\mathcal{X}} \lambda_0 d\mu \right].
\end{equation}
Finally, using the similar argument as the proof of Proposition \ref{prop:dualisdual_appen}, we obtain $\lambda = V \,\, \rho^\star_t$-a.s. and consequently, $u = -\nabla V$.
Finally, by substituting $V$ into $\lambda$, we obtain
\begin{equation} \label{eq:inner_loop_induce4_appen}
    \inf_{\rho} \left[ \int^1_0 \int_{\mathcal{X}} \left(\partial_t V - \frac{1}{2} \lVert \nabla V \rVert^2 + \frac{\sigma^2}{2} \Delta V \right) \rho_t dx dt  \right] + \int_{\mathcal{X}} V_0 d\mu.
\end{equation}

\paragraph{Conditional Sampling}
In this paragraph, we justify the sampling $x_{t+\Delta t} \sim \mathbb{P}_{t,1} \left(\cdot | x_t, \hat{y} \right)$ in line 5 of Algorithm \ref{alg:em_euot}.
Notice that we are sampling with the assumption of reciprocal property:
\begin{equation}
    \mathbb{P}_{t+\Delta t| t,1} \left(\cdot | x_t, \hat{y} \right) = \mathbb{Q}_{t+\Delta t|t,1} \left(\cdot | x_t, \hat{y} \right).
\end{equation}
% In other words, we are assuming that $(x_t, \hat{y})$ is sampled from the optimal (SB) pair between $\mathbb{P}_t$ and $\mathbb{P}_1$.
% Note that the SB problem between $\mathbb{P}_t = \rho_t$ and $\mathbb{P}_1 = \rho_1$ for time interval $[t,1]$ can be written as minimization problem of $D_{\text{KL}}(\mathbb{P}_{[t,1]}| \mathbb{Q}_{[t,1]})$.
Since $\mathbb{Q}_{t,1}$ is a Gaussian distribution of variance $\sigma^2 (1-t)$, the variance of $\mathbb{Q}_{t+\Delta t|t,1}$ is
$$
    \sigma^2 (1-t) \times \frac{\Delta t}{1-t} \times \frac{1 - t - \Delta t}{1-t} = \sigma^2 \frac{\Delta t(1-t-\Delta t)}{(1-t)^2}.
$$
Then, the average of $x_{t+\Delta t}$ is 
\begin{equation}
     \mathbb{P}_{t+\Delta t| t,1} (x_{t+\Delta t}| x_t, \hat{y}) = \mathcal{N}\left( x_{t+\Delta t} \bigg| \frac{\Delta t}{1-t} x_t + \frac{1-t-\Delta t}{1-t} \hat{y}, \ \sigma^2 \frac{\Delta t (1-t-\Delta t)}{(1-t)^2} \right).
\end{equation}

\section{Connection to Related Works} \label{appen:relatedworks}
In this section, we clarify the connection of our method to various existing OT algorithms.

\subsection{Connection to Wasserstein Lagrangian Flow}
In this section, we clarify the relationship of our approach compared to \citet{wlf}. 
In \cite{wlf}, they suggests a general framework for handling dynamical optimal transport problems for various cost functionals where the marginal distributions are fixed for some timesteps $\{0=t_0 , \ \dots, t_{N-1}, t_N=1\}$. 
The algorithm is derived by leveraging the Lagrangian dual formulations of these problems. Technically, this algorithm alternately updates the probability density $\rho_t$ and value function $V$.
It is easy to show that when the methodology proposed in \cite{wlf} is reduced to our problem, it becomes a max-min problem of Dual II (Proposition \ref{theorem:lagrangian_euot_appen}).

In our approach, we similarly derive a the objective for the density $\rho$ Appendix \ref{appen:derivations} by leveraging the lagrangian dual formulation. 
The key difference to the objective derived in \cite{wlf} lies in our active incorporation of optimality conditions of our EUOT formulation, which allowed us to avoid the minimax objective on the HJB-like equation (Eq. \ref{eq:lagrangian_euot_appen}). We believe this difference contributes to the scalability of our network, though the further investigation in required to fully understand its impact.

Furthermore, there is a significant difference in the evaluation process. 
Unlike our approach, \cite{wlf} does not directly learn the transport map. 
Instead, they train the value function $v_\phi$ and generate samples by solving stochastic differential equation (SDE) dynamics using the gradient of $v_\phi$, requiring approximately 100 function evaluations (NFE) for the evaluation. In contrast, our method directly parametrizes the transport map, enabling one-step evaluation.

% since it does not directly learn transport map, 

% In the image generation task, our method highly outperforms \citet{wlf}. Moreover, by incorporating a reciprocal property, our method allows one-step generation, while \citet{wlf} requires near 100 NFEs. 

% Our derivation slightly differs from  because the EUOT problem relaxes the terminal distribution by a f-divergence term. For completeness, here, we recall the Dual II problem (Eq. 18):
% \begin{equation}
%     \sup_{A} \inf_{\rho} \left[ \int_0^1 \int_{\mathcal{X}} \left( \partial_t A_t -\frac{1}{2} \| \nabla A_t \|^2 + \frac{\sigma^2}{2} \Delta A_t \right) d\rho_t dt + \int_{\mathcal{X}} A_0 d\mu  - \int_{\mathcal{X}} \alpha \Psi^* \left(\frac{A_1}{\alpha}\right) d\nu \right],
% \end{equation}

% Furthermore, there are also big difference in the engineeric point of view. we further develop the algorithm by utilizing the SOC viewpoint on Dual I (Eq. 17). As such, we remove the max-min objective on the first term and replace it with the expected value estimation and density minimization objective (see Eq. 22, 24). In the image generation task, our method highly outperforms \citet{wlf}. Moreover, by incorporating a reciprocal property, our method allows one-step generation, while \citet{wlf} requires near 100 NFEs. 

\subsection{Connection to Diffusion Schr\"{o}dinger Bridge Matching (DSBM)}
A key difference between our method and Schr\"{o}dinger bridge matching \cite{imf,gsbm} is that our approach eliminates burdensome simulations by directly parametrizing the transport plan $T_\theta$.
In \cite{imf}, the drift $u:=u_\theta$ of the SDE $dX^u_t = u dt + \sigma dW_t$ is learned in the following way: first, sample pairs \{$(X^u_0,X^u_1)$\} are obtained through SDE simulation. 
Then, sample an intermediate samples using the reciprocal property $X_t \sim \rho_{t|0,1}(\cdot| X^u_0, X^u_1)$. 
Note that this approach requires the heavy simulation of $X^u_1$ starting from $X^u_0$. 
Then, the drift $u_\theta$ is updated by Markovian projection (or action matching) using these triplets $\{ (X^u_0,X_t,X^u_1) \}$.

In contrast, our model directly obtain $\{(X_0, X_1)\}$ through the transport plan $T_\theta$. Then, we directly obtain $X_t \sim \rho_{t|0,1}(\cdot| X_0, X_1)$ by leveraging reciprocal property (Eq. 16). 
While our method reduces the simulation burden, it face some computational burden by the PINN-like objective.

\subsection{Other Related Works}
In this paragraph, we introduce several works that approximate the solution of EUOT between the continuous measures. It is challenging to directly address the EUOT problem in a continuous setting \cite{lightUOT}. Hence, several approaches utilize discrete minibatch EUOT approximations based on Sinkhorn Algorithm \cite{eyring2024unbalancedness, sinkhorn2}. These works employ minibatch EUOT plans to train flow matching \cite{lipman2022flow} or conditional flow matching \cite{tong2023conditional}. On the other hand, \cite{lightUOT} proposes a lightweight solver for the EUOT problem. This method is based on a Gaussian mixture approximation of the EUOT plan. These approaches rely on strong assumptions that the OT plan can be approximated by a minibatch EUOT plan or a Gaussian mixture. To the best of our knowledge, our work is the first attempt to propose a method for solving the continuous EUOT problem in Eq. \ref{eq:euot}, without any additional assumptions on the EUOT plan. Notably, our method also benefits from simulation-free training and offers one-step sample generation.

\section{Implementation Details} \label{appen:implementation_details}
Unless otherwise stated, the source distribution $\mu$ is a standard Gaussian distribution with the same dimension as the data (target) distribution $\nu$.
Moreover, computing $\Delta v_\phi$ in line 6 is burdensome, hence, we approximate it through Hutchinson-Skilling trace estimator \cite{hutchinson, skilling}.
Our implementation of the trace estimator follows the likelihood estimation implemented in \citet{scoresde}.

\subsection{2D Experiments}
\paragraph{Data Description} We obtain each data as follows:
\begin{itemize}
    \item \textbf{8-Gaussian}: For $m_i = 12 \left(\cos{\frac{i}{4}\pi},\sin{\frac{i}{4}\pi}\right)$ for $i=0, 1, \dots, 7$ and $\sigma=0.04$, the target distribution is defined as the mixture of $\mathcal{N}(m_i,\sigma^2)$ with an equal probability.
    % \item \textbf{8-Gaussian to 8-Gaussian}: For the source distribution, we use Gaussian mixture of $\mathcal{N}(m_i, \sigma^2 I)$ where $m_i = 2 \left(\cos{\frac{i}{4}\pi},\sin{\frac{i}{4}\pi}\right)$ for $i=0, 1, \dots, 7$ and $\sigma=0.02$. For the target distribution, we use $m_i = 4 \left(\cos{\frac{i}{4}\pi},\sin{\frac{i}{4}\pi}\right)$ for $i=0, 1, \dots, 7$ and $\sigma=0.02$.
    \item \textbf{Moon to Spiral}: We follow \citet{uotmsd}.
    \item \jm{\textbf{Gaussian Experiments}: We follow \citet{imf}, note that the data dimension is 50.}
\end{itemize}

\paragraph{Network Architectures}
\jm{
Let $y$ be concatenated variable of source data $x$ and auxiliary variable $z\sim \mathcal{N}(0,I)$.
The dimension of the auxiliary variable is set to be same as the dimension of data.
We parametrize the generator by $T_\theta(x, z) = x + t_\theta (y)$ where $t_\theta$ is the MLP of three hidden layers. Here, $y$ is the concatenation of $x$ and $z$.
Moreover, for the discriminator, we concatenate $x$ and $t$ and pass it through a 3-layered MLP.
We employ a hidden dimension of 256 and a SiLU activation function.
}

\paragraph{Training Hyperparameters}
\jm{
We trained for 120K iterations with the batch size of 1024. 
We set Adam optimizer with $(\beta_1, \beta_2 )=(0, 0.9)$, learning rate of $2\times 10^{-4}$ and $10^{-4}$ for the transport map $T_\theta$ and the potential $v_\phi$, respectively.
We use a cosine scheduler to gradually decrease the learning rate from $10^{-4}$ to $5\times 10^{-5}$.
We update the inner objective (the generator) three times for every single update of the outer objective.
% We used a gradient clip of $1.0$ for both the generator and discriminator.
We use the number of timesteps of $20$, $\alpha=1$, and uniform distribution for $\mathcal{T}$.
Moreover, we set $\lambda_G = 0.1, \ \lambda_D = 1$, $p=1$ and $\sigma = 0.8$.}

\paragraph{Discrete OT Solver}
We used the POT library \cite{pot} to obtain an accurate transport plan $\pi_{pot}$. 
We used 1024 training samples for each dataset in estimating $\pi_{pot}$ to sufficiently reduce the gap between the true continuous measure and the empirical measure.

\subsection{Implemenation on Image Datasets}
\paragraph{CIFAR-10 Generation}
We follow the generator $T_\theta$ architecture and the implementation of \citet{uotm}.
Moreover, we follow the discriminator architecture in \citet{rgm}. 
We trained for 300K iterations with the batch size of 256. 
We use Adam optimizer with $(\beta_1, \beta_2 )=(0, 0.9)$, learning rate of $10^{-4}$.
We use a cosine scheduler to gradually decrease the learning rate from $10^{-4}$ to $5\times 10^{-5}$.
We used a gradient clip of $1.0$ for both the generator and discriminator.
We use the number of timesteps of $20$, $\alpha=1$, $p=2$, $\sigma=0.1$, and linear time distribution $\mathcal{T}$.
Moreover, when $\Psi^*(x) = 2 \log(1+e^x) - 2 \log 2$ (\textbf{EUOT-Soft}), we set $\lambda_G = 5, \ \lambda_D = 1$.
When $\Psi^*(x) = 5 e^{x/5} - 5$ (\textbf{EUOT-KL}), we use $\lambda_G = 1, \ \lambda_D = 1$.
Note that $\Psi^*(x) = 5 e^{x/5} - 5$ is the convex conjugate of entropy function $\Psi$ for $5 D_{\text{KL}}$.
Unless otherwise stated, we consider $\Psi^*(x) = 2 \log(1+e^x) - 2 \log 2$.
\jm{To ensure stable training, we introduce the $R_1$ regularization term with a regularization coefficient of $\lambda=2.5$ during the update of the potential $v_\phi$. n other words, we add  $\lambda \lVert \nabla v_\phi(t, x) \rVert^2$ on line 8 of Alg, \ref{alg:em_euot}.}

\paragraph{Image-to-image Translation}
\jm{We follow the generator $T_\theta$ architecture of \cite{uotm}. For the discriminator, we follow the largest discriminator used in \cite{uotm} when training CelebA-HQ 256x256 dataset. We train for 30K iterations with the batch size of 64. We use Adam optimizer with  $(\beta_1, \beta_2 )=(0, 0.9)$. The learning rate of generator and discriminator is  $2\times 10^{-4}$ and $10^{-4}$, respectively. We use the number of timesteps of 20, $\alpha=1$, $p=2$, $\sigma=0.5$, $\lambda_G = \lambda_D = 1$, and uniform time distribution $\mathcal{T}$.
Note that in the image translation experiments, we do not use clip nor $R_1$ regularization.}

\paragraph{Evaluation Metric}
For the evaluation of CIFAR10, we used 50,000 generated samples to measure FID \citep{stylegan} scores.
\jm{
For the Wild$\rightarrow$Cat experiments, we compute FID statistics based on train dataset of the target dataset. Then, we generate samples $T_\theta(x)$ from each source sample $x$ from test dataset.
We sampled multiple generated samples from each source. Then, we calculate the FID based on this generated samples. Note that this follows the implementation of \cite{sb-flow}. 
For CelebA experiment, we follow the metric used in \cite{not}. Specifically, we compute FID through test datasets.
}

\section{Additional Results} \label{appen:results}

\begin{figure}[ht]
    \centering
    \includegraphics[width=.86\textwidth]{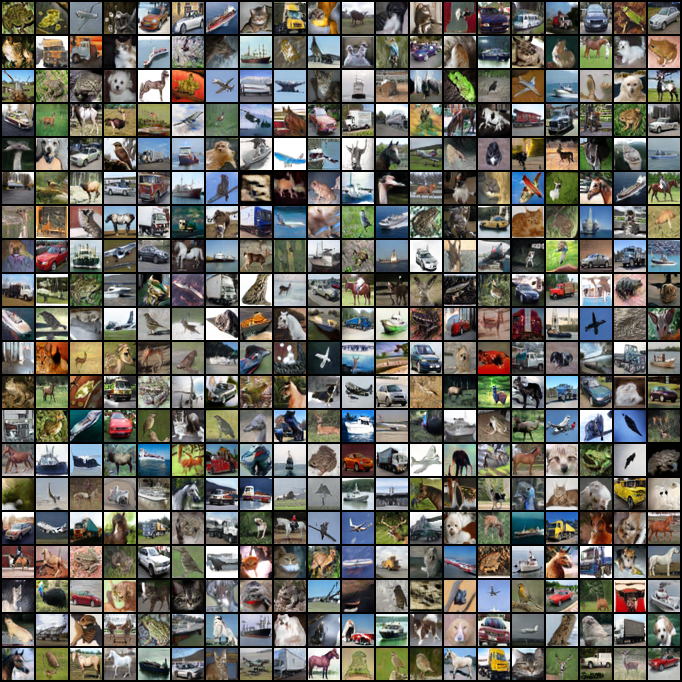}
    \caption{Generated samples from our model trained on CIFAR-10 for $\sigma=0.1$.}
\end{figure}

\begin{figure}[ht]
    \centering
    \includegraphics[width=.86\textwidth]{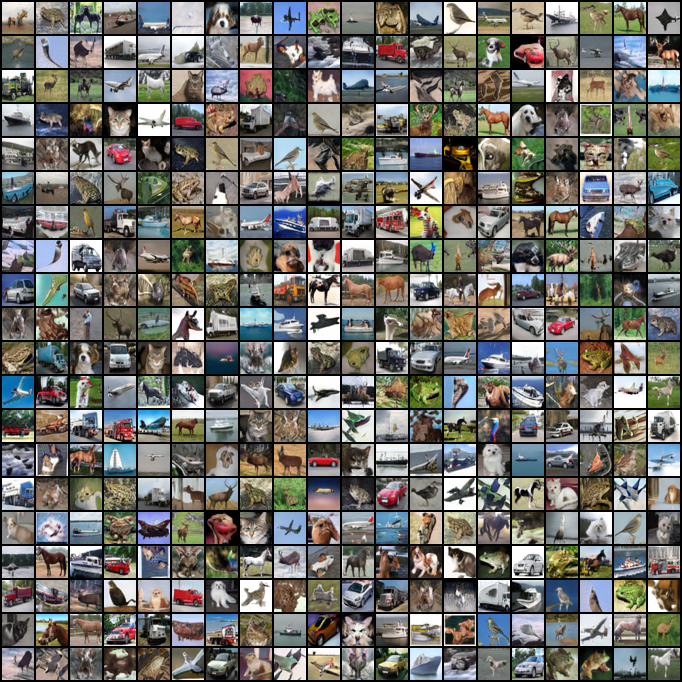}
    \caption{Generated samples from our model trained on CIFAR-10 for $\sigma=0.5$.}
\end{figure}

\begin{figure}[ht]
    \centering
    \includegraphics[width=.86\textwidth]{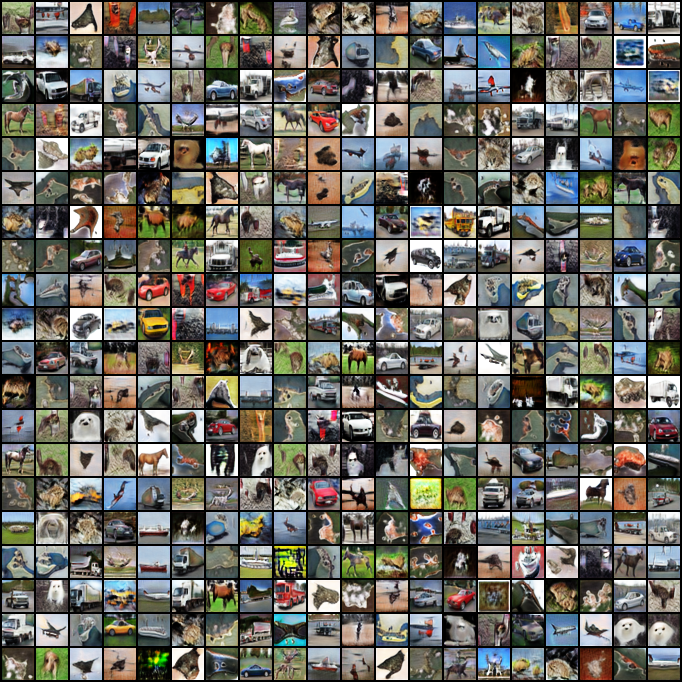}
    \caption{Generated samples from our model trained on CIFAR-10 for $\sigma=1.0$.}
\end{figure}

% \clearpage
% \newpage
\begin{figure}[ht]
    \centering
    % \hfill
    \includegraphics[width=.495\textwidth]{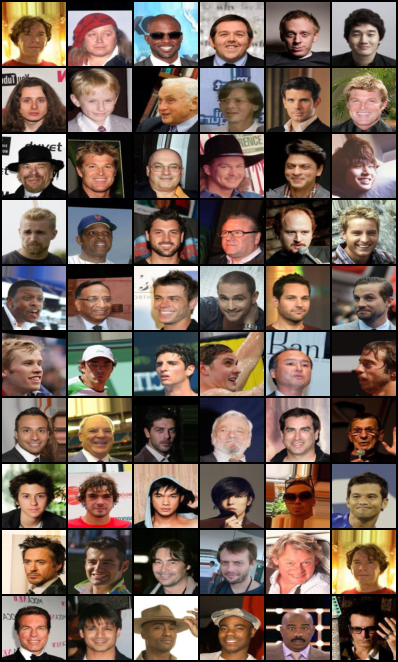}
    \hfill
    \includegraphics[width=.495\textwidth]{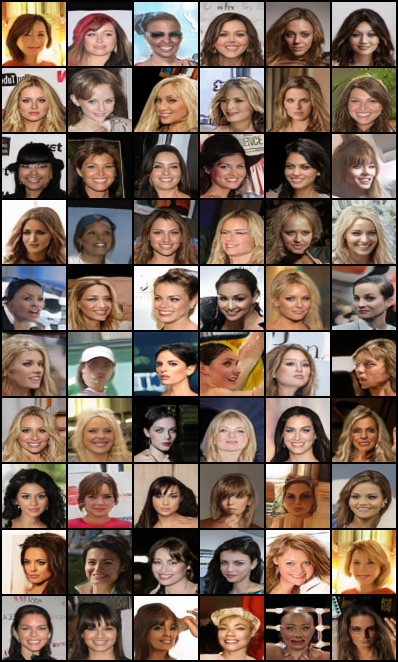}
    % \hfill
    \caption{Unpaired $male \rightarrow female$ translation for $64 \times 64$ CelebA images.}
\end{figure}

\begin{figure}[ht]
    \centering
    % \hfill
    \includegraphics[width=.495\textwidth]{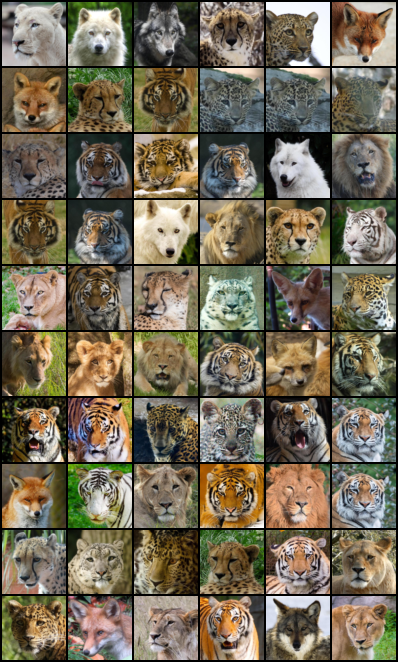}
    \hfill
    \includegraphics[width=.495\textwidth]{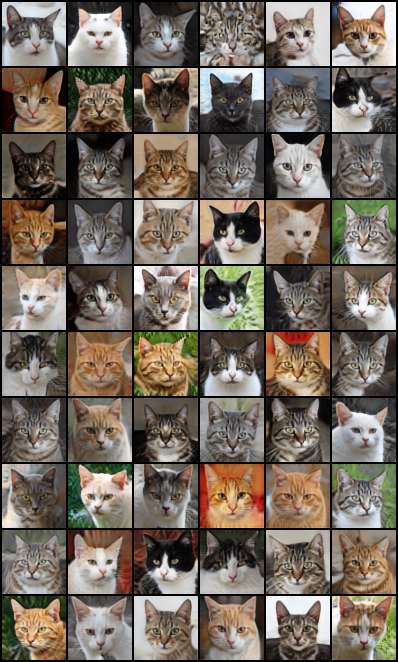}
    % \hfill
    \caption{Unpaired $wild \rightarrow cat$ translation for $64 \times 64$ CelebA images.}
\end{figure}

\end{document}